%% file: main.tex
\documentclass[11pt, oneside]{article}
\usepackage[margin=1.0in]{geometry}
\usepackage{graphicx}		
\usepackage{amssymb}
\usepackage{amsmath}
\usepackage{amsthm}
\usepackage{mathtools}
\usepackage{hyperref}
\usepackage{enumitem}
\usepackage{verbatim}
\usepackage{xcolor}
\usepackage[utf8]{inputenc}
\usepackage[T1]{fontenc}
\usepackage{microtype}
\usepackage{natbib}
\usepackage{physics}
\allowdisplaybreaks
\usepackage{booktabs}       % professional-quality tables
\usepackage{amsfonts}       % blackboard math symbols
\usepackage{nicefrac}       % compact symbols for 1/2, etc.
\usepackage{tabularx}
\usepackage{multirow}
\usepackage{multicol}
\usepackage{longtable}
\usepackage{rotating}
\usepackage{array}
\usepackage{afterpage}
\usepackage{tikz}
\usepackage{pgfplots}
\pgfplotsset{compat=1.5}

\newtheorem{theorem}{Theorem}
\newtheorem{definition}[theorem]{Definition}

\newtheorem{lemma}[theorem]{Lemma}
\newtheorem{claim}[theorem]{Claim}
\newtheorem{corollary}[theorem]{Corollary}
\numberwithin{theorem}{section}

\newcommand{\ex}[2]{{\ifx&#1& \mathbb{E} \else \underset{#1}{\mathbb{E}} \fi \left[#2\right]}}
\newcommand{\pr}[2]{{\ifx&#1& \mathbb{P} \else \underset{#1}{\mathbb{P}} \fi \left[#2\right]}}

\newcommand{\eps}{\varepsilon}

\newcommand{\argmax}{\textup{argmax}}
\newcommand{\OPT}{\texttt{OPT}}
\newcommand{\B}{\mathcal{B}}

\newcommand{\D}{\mathcal{D}}
\newcommand{\I}{\mathcal{I}}
\newcommand{\M}{\mathcal{M}}
\newcommand{\N}{\mathcal{N}}

\newcommand{\U}{\mathcal{U}}

\newcommand{\tw}{\tilde{w}}

\let\originalleft\left
\let\originalright\right
\renewcommand{\left}{\mathopen{}\mathclose\bgroup\originalleft}
\renewcommand{\right}{\aftergroup\egroup\originalright}

\usepackage{bbm}
\usepackage{float}
\usepackage[]{algorithm}
\usepackage[noend]{algpseudocode}
\makeatletter
\def\BState{\State\hskip-\ALG@thistlm}
\makeatother

\title{Differentially Private Decomposable Submodular Maximization}
\author{Anamay Chaturvedi\thanks{Khoury College of Computer Sciences, Northeastern University \dotfill \texttt{chaturvedi.a@northeastern.edu}} \and Huy L\^{e} Nguy\~{\^{e}}n \thanks{Khoury College of Computer Sciences, Northeastern University \dotfill \texttt{hlnguyen@cs.princeton.edu}}\and Lydia Zakynthinou\thanks{Khoury College of Computer Sciences, Northeastern University \dotfill \texttt{zakynthinou.l@northeastern.edu}}}
\date{\today}

\begin{document}
\maketitle
\begin{abstract}
We study the problem of differentially private constrained maximization of decomposable submodular functions. A submodular function is decomposable if it takes the form of a sum of submodular functions. The special case of maximizing a monotone, decomposable submodular function under cardinality constraints is known as the Combinatorial Public Projects (CPP) problem \citep{PapadimitriouSS08}. Previous work by \citet{GuptaLMRT10} gave a differentially private algorithm for the CPP problem. 

We extend this work by designing differentially private algorithms for both monotone and non-monotone decomposable submodular maximization under general matroid constraints, with competitive utility guarantees. 
%Our algorithms have a $(1-1/e)$ and $1/e$ multiplicative approximation factor for monotone and non-monotone objectives, respectively. 
We complement our theoretical bounds with experiments demonstrating empirical performance, which improves over the differentially private algorithms for the general case of submodular maximization and is close to the performance of non-private algorithms.
\end{abstract}

\input{intro}
\input{prelim}
\input{monotone}
\input{nonmonotone}

\input{experiments}

\section*{Acknowledgements}
The authors were supported by NSF grants CCF 1909314 and CCF 1750716. LZ was also supported by NSF grants CCF-1718088, CCF-1750640.

\bibliographystyle{plainnat}
\bibliography{biblio}
	
\end{document}

%% file: intro.tex
\section{Introduction}\label{sec:intro}
A set function $f:2^{\N} \rightarrow \mathbb{R}$ is submodular if it satisfies the following property of diminishing marginal returns: for all sets $S\subseteq T\subseteq\N$ and every element $u\in\N\setminus T$, $f(S \cup \{u\})-f(S) \geq f(T \cup \{u\}) -f(T)$. In other words, the marginal contribution of any element $u$ to the value of the function $f(S)$ diminishes as the input set $S$ increases in size. 

The theory of submodular maximization provides a unifying framework which captures many combinatorial optimization problems, including Max Cut, Max $r$-Cover, Facility Location, and Generalized Assignment problems. 
Optimization problems involving the maximization of a submodular objective function arise naturally in many different applications, which span a wide range of fields such as computer vision, operations research, electrical networks, and combinatorial optimization (see for example~\citep{Narayanan97, Fujishige05, Schrijver03}).
%examples for computer vision \citep{Hochbaum01, BoykovVZ01}
Furthermore, submodular functions are extensively used in economics (e.g., in the problem of welfare maximization in combinatorial auctions~\citep{DobzinskiS06, Feige06, FeigeV06, Vondrak08}) because their diminishing returns property captures the preferences of agents for substitutable goods. 

Recently, submodular maximization has found numerous applications to problems in machine learning~\citep{ KawaharaKTB09}. Such applications include the problems of influence maximization in social networks~\citep{KempeKT03, BorgsBCL14, BorodinJLY17}, result diversification in recommender systems~\citep{ParambathUG16}, feature selection for classification~\citep{KrauseG05}, dictionary selection~\citep{KrauseC10}, document and corpus summarization~\citep{LinB11, KirchhoffB14, SiposSSPT12}, crowd teaching~\citep{SinglaBBKK14}, and exemplar-based clustering~\citep{DueckF07, GomesK10}.
%pac learning \citep{NarasimhanB04}, 
 
In all these settings, the goal is to optimize a submodular function $f$ subject to certain constraints. These constraints can be simple, such as cardinality constraints (i.e., $\max\{f(S) : |S| \leq r\}$) or more general matroid constraints (e.g., a partition matroid $\max\{f (S ): |S \cup P_i | \leq 1, \forall i\}$, where $P_i\subset\N$ are disjoint sets). We give the formal definition of matroid constraints in Section~\ref{sec:matroid}.

In most applications, these machine learning tools are applied to users' sensitive data, causing privacy concerns to become increasingly important, and motivating the study of private submodular maximization. 
Differential Privacy (DP)~\citep{DworkMNS06} has been widely-accepted as a robust mathematical guarantee that a model produced by a machine learning algorithm does not reveal sensitive, personal information contained in the training data. Informally, DP ensures that the output of the algorithm will not change significantly if one individual's input is swapped with another's. We provide the formal definition in Section~\ref{sec:privacy}.

Notably,~\citet{MitrovicBKK17} gave differentially private algorithms for monotone and non-monotone submodular maximization under cardinality, matroid, and p-extendible system constraints. \citet{GuptaLMRT10} studied a variety of combinatorial optimization problems under differential privacy and, in particular, gave a differentially private algorithm for the Combinatorial Public Projects (CPP) problem introduced by~\citet{PapadimitriouSS08}. This is a special case of monotone submodular maximization under cardinality constraints, as the objective function $f$ is \emph{decomposable} (also known as \emph{Sum-of-Submodular}), i.e., it is the sum of submodular functions. The goal is then to find the set that maximizes $f$, while protecting the private valuation functions of the individual users.

Decomposable submodular functions encompass many of the examples of submodular functions studied in the context of machine learning as well as welfare maximization. 
%In particular, the paradigm of empirical risk minimization, which is ubiquitous in machine learning, fits into this model: the goal is to find a parameter which will minimize the sum of the loss functions of the agents, each of which depends on the input datapoint that the agent contributes to the sample.
In the latter, each agent has a valuation function over sets, and the goal is to maximize the sum of the valuations of the agents, i.e., the ``social welfare''. The valuation functions are often submodular as they exhibit the diminishing returns property.
In machine learning, data summarization, where the goal is to select a representative subset of elements of small size, falls into this setting and has numerous applications, including exemplar-based clustering, image summarization, recommender systems, active set selection, and document and corpus summarization. The line of work of \citet{MirzasoleimanBK16, MirzasoleimanKSK16, MirzasoleimanZK16} studies decomposable submbodular maximization under matroid and $p$-systems constraints in various data summarization settings and takes different approaches to user privacy.

\subsection{Our contributions}

We focus on the problem of maximization of a decomposable submodular function under matroid constraints.
Concretely, suppose there exists a set of agents $D$ of size $m$ and a ground set of elements $\N$ of size $n$. We assume that each agent $I$ has a submodular function $f_I : 2^{\N} \to [0,\lambda]$, and the goal is to find the subset $S\subseteq\N$ maximizing $f(S) = \sum_{I \in D} f_I(S)$ subject to a matroid constraint $\M=(\N,\I)$ of rank $r$, under differential privacy.

We provide two algorithms for the maximization of a decomposable submodular function under matroid constraints both for the case of monotone and non-monotone functions. Our algorithms have competitive utility guarantees, as presented in the next theorems, where $f(\OPT)$ denotes the value of the optimal non-private solution. The complete results for monotone and non-monotone decomposable functions can be found in Sections~\ref{sec:monotone} and~\ref{sec:nonmonotone}, respectively.

\begin{theorem}[Informal]\label{th:main-monotone-intro}
If $f$ is a monotone $\lambda$-decomposable function, then there exists an $(\eps,\delta)$-differentially private algorithm for maximization of $f$ under matroid constraints of rank $r$, which, given parameters $\eta, \gamma, \eps,\delta$, returns a set $S$ such that, with probability $1-\gamma$, \[\ex{}{f(S)}\geq (1-1/e-O(\eta))f(\OPT)-O\left(\frac{\lambda r}{\eta\eps}\log\frac{nr}{\eta\gamma}\cdot\log\frac{1}{\delta}\right).\]
\end{theorem}

\begin{theorem}[Informal]\label{th:main-nonmonotone-intro}
If $f$ is a non-monotone $\lambda$-decomposable function, then there exists an $(\eps,\delta)$-differentially private algorithm for maximization of $f$ under matroid constraints of rank $r$, which, given parameters $\eta, \gamma, \eps,\delta$, returns a set $S$ such that, with probability $1-\gamma$, \[\ex{}{f(S)}\geq (1/e-O(\eta))f(\OPT)-O\left(\frac{\lambda r}{\eta\eps}\log\frac{nr}{\eta\gamma}\cdot\log\frac{1}{\delta}\right).\]
\end{theorem}

Our results extend the results of~\citet{GuptaLMRT10} from cardinality to matroid constraints, as well as non-monotone functions. We note that the multiplicative factor of our utility guarantee for the monotone case is optimal for the non-private version of the problem and the additive factor is optimal for any $\eps$-differentially private algorithm for the problem (see lower bound by~\citealt[][Thm.~8.5]{GuptaLMRT10}). %$O(\frac{r\log(n/r)}{\eps})$ 
Our solution exhibits a tradeoff between the multiplicative and the additive factor of the utility through the parameter $\eta$, which depends on the chosen number of rounds of the algorithm and is a small constant.

In comparison, the general case of submodular function maximization assumes functions of bounded \emph{sensitivity}, that is, $\max_S\max_{A, B}{|f_A(S)-f_B(S)|}\leq \lambda$ for $A,B$ sets of agents that differ in at most one agent. %Bounded sensitivity is essential in the application of fundamental differentially private mechanisms.
The decomposability assumption allows us to improve on the utility guarantees of the general case of the maximization of monotone and non-monotone submodular $\lambda$-sensitive functions, studied by~\citet{MitrovicBKK17}, in our multiplicative and additive factor. 
We remark that~\citet{MitrovicBKK17} also note that using their general greedy algorithm for monotone submodular maximization under matroid constraints with the analysis of~\citet{GuptaLMRT10} yields a result for decomposable functions with improved additive error compared to the general case.

In proving our results, we also fix a lemma that is essential in the privacy analysis of the CPP problem of~\citet{GuptaLMRT10} and, in turn, in the result for decomposable monotone submodular maximization under matroid constraints of ~\citep{MitrovicBKK17} that was mentioned above, which allows for the improved additive factor in the utility of the algorithms.

Our contributions are summarized Table~\ref{table}.

\defcitealias{MitrovicBKK17}{MBKK17}
\defcitealias{GuptaLMRT10}{GLMRT10}

%\afterpage{
\LTcapwidth=\textwidth
\setlength\extrarowheight{8pt}
\begin{longtable}[c]{|c|c|c|c|}
\cline{3-4}
\multicolumn{2}{c|}{}&
$r-$Cardinality &
Matroid (rank $r$)\\
 \hline 
 \multirow{4}{*}{\begin{sideways}$\lambda$-decomposable\end{sideways}} &   
\multicolumn{1}{c|}{\multirow{2}{*}{Monotone}} & $\left(1-\frac{1}{e}\right)f(\OPT)-\frac{r\lambda}{\eps}\log n$  & $\left(1-\frac{1}{e}-\eta\right)f(\OPT)-\frac{r\lambda}{\eta\eps}\log n$ [This work]  \\
& & \citepalias{GuptaLMRT10} & $\frac{1}{2}f(\OPT)-\frac{r\lambda}{\eps}\log n$ \citepalias{MitrovicBKK17}\\
\cline{2-4}
& \multicolumn{1}{c|}{\multirow{2}{*}{Non-monotone}} &
 $\left(\frac{1}{e}-\eta\right)f(\OPT)-\frac{r\lambda}{\eta\eps}\log n$ & $\left(\frac{1}{e}-\eta\right)f(\OPT)-\frac{r\lambda}{\eta\eps}\log n$   \\
&  &  [This work]  & [This work]   \\
\hline   
\multirow{4}{*}{\begin{sideways}$\lambda$-sensitive\end{sideways}} &   
\multicolumn{1}{c|}{\multirow{2}{*}{Monotone}} &
$\left(1-\frac{1}{e}\right)f(\OPT)-\frac{r^{3/2}\lambda}{\eps}\log n$ &  $\frac{1}{2}f(\OPT)-\frac{r^{3/2}\lambda}{\eps}\log n$ \\
& & \citepalias{MitrovicBKK17}  &   \citepalias{MitrovicBKK17}  \\
\cline{2-4}
& \multicolumn{1}{c|}{\multirow{2}{*}{Non-monotone}} &
 $\frac{1}{e}\left(1-\frac{1}{e}\right)f(\OPT)-\frac{r^{3/2}\lambda}{\eps}\log n$ & $-$   \\
&  & \citepalias{MitrovicBKK17} &   \\
\hline
\caption{Expected utility guarantees of $(\eps,\delta)$-differentially private submodular maximization algorithms. All results omit any $\log\frac{1}{\delta}$ and constant factors.}\label{table}
\end{longtable}
%}

We complement our theoretical bounds with experiments on a dataset of Uber pickups in Manhattan~\citep{Uber} in Section~\ref{sec:experiments}. We show that our algorithms perform better than the more general algorithms of~\citep{MitrovicBKK17} for monotone submodular maximization under matroid and cardinality constraints, and they are competitive with the non-private greedy algorithm.

\subsection{Related work}
\paragraph{Submodular maximization}
 There is a vast literature on submodular maximization (see~\citep{BuchbinderF18} for a survey), for which the greedy technique has been a dominant approach. \citet{NemhauserWF78} introduced the basic greedy algorithm for the maximization of a monotone submodular function, that iteratively builds a solution by choosing the item with the largest marginal gain with respect to the set of previously selected items. This algorithm achieves a $\left(1-\frac{1}{e}\right)$-approximation for a cardinality constraint (which is optimal~\citep{RazS97}) and a $\frac{1}{2}$-approximation for a matroid constraint.

\citet{CalinescuCPV11} developed a framework based on continuous
optimization and rounding that led to an optimal $\left(1-\frac{1}{e}\right)$-approximation for the problem. The approach is to turn the discrete optimization problem of maximizing a submodular function $f$ subject
to a matroid constraint into a continuous optimization problem of maximizing the
multilinear extension $F$ of $f$ (a continuous function that extends $f$) subject to the matroid polytope
(a convex polytope whose vertices are the feasible integral solutions). The continuous optimization
problem can be solved approximately within a $\left(1-\frac{1}{e}\right)$ factor using a \emph{Continuous Greedy} algorithm~\citep{Vondrak08}.

In each round $t=1,\ldots,T$, the Continuous Greedy algorithm estimates the marginal gains of each element $u$ with respect to the current fractional solution $y^{(t)}$ within a small sampling error, that is, it estimates $F(y^{(t)}\wedge \mathbf{1}_{u}) -F(y^{(t)})=\ex{}{f(R(y^{(t)})\cup\{u\})-f(R(y^{(t)}))}$, where $R(y)$ is a random set which contains each element $v$ independently with probability $y_v$. The algorithm then finds an independent set, $B^{(t)}$, of the matroid, which maximizes the sum of the estimated marginal gains of the items. It then updates the current fractional solution by taking a small step $\eta=1/T$ in the direction of the selected set: $y^{(t+1)}=y^{(t)}+\eta \mathbf{1}_{B^{(t)}}$. The final fractional solution $y^{(T)}$ can then be rounded to an integral one without loss~\citep{ChekuriVZ10}.

The \emph{Measured Continuous Greedy} algorithm introduced by~\citet{FeldmanNS11} is a unifying variant of the continuous greedy algorithm, which increases the coordinates of its fractional solution more slowly, and achieves a $\frac{1}{e}$-approximation factor for the more general case of non-monotone submodular functions. This is not the optimal constant factor approximation for the problem~\citep{BuchbinderF19, EneN16}, but the structure of the algorithm helps in its private adaptation.

\paragraph{Private submodular maximization}
The private algorithms of~\citep{MitrovicBKK17} and~\citep{GuptaLMRT10} are based on the discrete greedy algorithm, where the greedy step of selecting an item in each round is instead implemented via the differentially private Exponential Mechanism of~\citet{McSherryT07}, which guarantees that the item selected in each round is almost as good as the true marginal gain maximizer, with high probability.
In general, by the advanced composition property of DP~\citep{DworkRV10}, $r$ consecutive runs of an $\eps$-differentially private algorithm lead to a cumulative privacy guarantee of the order of $\sqrt{r}\eps$.
Remarkably, for the case of decomposable monotone submodular functions,~\citet{GuptaLMRT10} show that the privacy guarantee of $r$ rounds is, up to constant factors, the same as that of a single run of the Exponential Mechanism. 

The main idea of this improvement is the following. Let $A, B$ be two sets of agents which differ in the individual $I$, as $A=B\cup\{I\}$. The privacy loss of the algorithm is bounded by the sum over the rounds of the expected marginal gains of each item with respect to the valuation function of agent $I$, where the expected value is calculated over a distribution that depends on the valuation functions of the rest of the agents $B$. More formally, the privacy loss is bounded by $\sum_{i=1}^r \ex{u}{f_I(S_{i-1}\cup \{u\})-f_I(S_{i-1})}$. By a key lemma, whose proof we fix and state in Section~\ref{sec:monotone}, this is bounded by a function of the sum of the \emph{realized} marginal gains $\sum_{i=1}^r [f_I(S_{i-1}\cup \{u_i\})-f_I(S_{i-1})]=\sum_{i=1}^r [f_I(S_{i})-f_I(S_{i-1})]=f_I(S)-f_I(\emptyset)$, which in turn is bounded by $\lambda$. Note that it is important in this argument that the sum telescopes to the total utility gain of the submodular function $f_I$.

Finally, we also note that, in principle, differentially private submodular optimization is related to submodular maximization in the presence of noise~\citep{HassidimS17}. However, the structure of the noise is of multiplicative nature, so it is not clear how these algorithms could be applicable.

\subsection{Techniques}
Our algorithms for the monotone and non-monotone problems are a private adaptation of the Continuous and Measured Continuous Greedy algorithms, respectively.
They both use the Exponential Mechanism to greedily find an independent set $B^{(t,r)}$ in each round $t$ and update with this set the current fractional solution.
Our privacy analysis is based on the technique of~\citep{GuptaLMRT10}. 

Let us now explain the main challenges in its application. Recall that we use the continuous greedy algorithm to achieve the optimal multiplicative guarantee for the monotone case, which means that instead of calculating the marginal gain of each element $u\in\N$ in each round with respect to $f$, we have to estimate it with respect to $f$'s multilinear extension $F$. That is, we estimate the quantinty $F(y^{(t,i-1)}\wedge \mathbf{1}_{u}) -F(y^{(t,i-1)})=\ex{R}{f(R(y^{(t,i-1)})\cup\{u\})-f(R(y^{(t,i-1)})}$. 

First, since the random sets $R(y^{(t,i-1)})$ used for this estimation are drawn independently in each round, the final sum of estimated marginal gains of $f_I$ is not a telescoping sum. Second, even if instead we use concentration to argue that the final sum is close to the true marginal gains with respect to $F_I$, this would lead to a final telescoping sum in the order of $T\lambda$. To overcome both these problems:
\begin{enumerate}
\item We choose the smoother marginal gains $F(y^{(t,i-1)}+\eta\mathbf{1}_{u}) -F(y^{(t,i-1)})$, so that the realized marginal gain is $F_I(y^{(t,i)}) -F_I(y^{(t,i-1)})$, which, by concentration, leads to the telescoping sum $F_I(y^{(T,r)})-F_I(y^{(1,0)})$ up to the sampling error term. However, this modification is not enough, as the sampling error would be on the order of $m$, the number of agents. In the interesting regimes, this is a large enough so that we would want to avoid any dependence on $m$ in the utility or sample complexity.
\item We introduce a function $G:[0,1]^{\N}\rightarrow \mathbb{R}$. The function $G$ is not submodular but it is a proxy for $F$. To construct $G$, we draw uniform vectors $r^j\in[0,1]^{\N}$ in the beginning of the algorithm, and define $G(x)$ to be the average over samples $f(\{u\in\N: r^j_u<x_u\})$. Therefore, $G_I(x)$ is always bounded by $\lambda$ and the sum of estimated marginal gains of agent $I$ telescopes to $G_I(y^{(T,r)})-G_I(y^{(1,0)})\leq \lambda$. Hence, $G$'s sampling error only affects the utility of the algorithm.%\footnote{We note that since $G$ is constructed on the same random vectors and it must be a good approximation for all possible fractional solutions over the course of the algorithm (these are at most $n^{rT}$), the number of oracle calls required is larger than the one used in the discretized continuous algorithms. Nonetheless, it does not depend on $m$.}
\end{enumerate}

Finally, further applying this technique to non-monotone functions requires a bound on the sum of the \emph{absolute} marginal gains (equivalently, the ``total movement'') of a non-monotone submodular function on non-decreasing inputs. This bound does not hold in general for any non-monotone function, but it holds for submodular functions, as we show in Section~\ref{sec:nonmonotone}.

%% file: prelim.tex
\section{Preliminaries}\label{sec:prelim}
Let $\N$ be the set of elements and let $|\N|=n$.
We denote by $\mathbf{1}_{S} \in [0,1]^n$ the indicator vector of the set $S$, that is, a vector $y=\mathbf{1}_{S}$ if $y_u=1$ for $u\in S$ and $y_u=0$ otherwise. We will slightly abuse notation and denote $\mathbf{1}_{\{u\}}$ by $\mathbf{1}_u$ and $\mathbf{1}_{\N\setminus \{u\}}$ by $\mathbf{1}_{\bar{u}}$. 	
Let $\U$ denote the uniform distribution over $[0,1]$. We will write $r\gets \U^n$ for the random vector drawn from the product of uniform distributions.
We write $\log$ for the natural logarithm.
\subsection{Submodular functions}
We consider decomposable submodular functions $f(S)=\sum_{I\in D} f_I(S)$ where $D$ is a set of agents of size $|D|=m$. More formally, we have the following definitions.
	
\begin{definition}\label{def:submodular}
	A set function $f:2^{\N}\rightarrow \mathbb{R}_{+}$ is \emph{submodular} if for all sets $S\subseteq T\subseteq \N$ and every element $u\in\N\setminus T$, we have $f(S\cup \{u\})-f(S)\geq f(T\cup \{u\})-f(T)$.
\end{definition}
	
Note that we consider only non-negative submodular functions. 
Moreover, if $f(S)\leq f(T)$ for all $S\subseteq T\subseteq\N$, we say that $f$ is monotone.
\begin{definition}\label{def:decomposable}
	A function $f:2^{\N}\rightarrow \mathbb{R}$ is \emph{$\lambda$-decomposable} if $f(S)=\sum_{I\in D}f_I(S)~\forall S\subseteq \N$, for submodular functions $f_I: 2^{\N}\rightarrow [0,\lambda]$.
\end{definition}
	
For ease of notation, we will assume $\lambda=1$. This is equivalent to adding a pre-processing step to scale the function by $1/\lambda$, which only affects the additive error of our approximation guarantees. By multiplying the additive error by $\lambda$, we can retrieve the general result for $\lambda$-decomposable functions.
	
\begin{definition}\label{def:multilinear}
	The multilinear extension of a submodular function $f$ is $F(x)=\ex{}{f(R(x))}$ where $R(x)$ is a random set which contains each element $u\in \N$ independently with probability $x_u$. Specifically, we can write \[F(x)=\sum_{S\subseteq \N} f(S) \prod_{u\in S} x_u \prod_{u\notin S} (1-x_u).\]
\end{definition}
By linearity of expectation, if $f$ is a decomposable submodular function, its multilinear extension satisfies $F(x)=\sum_{I\in D}F_I(x)$, where $F_I$ is the multilinear extension of $f_I$ for any agent $I$.
%Assume that we have an oracle for the multilinear extension $F$. We will later replace this assumption by an estimation process and adjust our results to account for the sampling error.
	
\subsection{Matroid constraints}\label{sec:matroid}
A matroid is formally defined as follows.	
\begin{definition}\label{def:matroid}
	A matroid is a pair $\M=(\N, \I)$ where $\I\in 2^{\N}$ is a collection of independent sets, satisfying
	\begin{enumerate}[label=(\roman*)]
	\item $A\subseteq B$, $B\in\I$ $\Rightarrow A\in \I$, and
	\item $A, B\in \I$, $|A|<|B|$ $\Rightarrow$ $\exists i\in B\setminus A$ such that $A\cup\{i\}\in \I$.
	\end{enumerate}
\end{definition}
	
Let $P(\M)\subseteq [0,1]^{\N}$ be the convex body which contains the characteristic vectors of the sets in the matroid constraint set $\I$. We assume it is down-closed, that is for $x\leq x'$ (where ``$\leq$'' denotes coordinate-wise comparison) if $x'\in P(\M)$ then $x\in P(\M)$.
	 
A maximal independent set is called \textit{a basis} for the matroid. We denote the set of bases by $\B$. Also, we let $r=\max_{B\in\I} |B|\leq n$ be the maximum size of a basis in $\I$, i.e., the \emph{rank} of the matroid. We state the following useful property of the matroid's bases.
	
\begin{lemma}[{\citealt[][Corollary 39.12A]{Schrijver03}}]\label{lem:mappingtoOPT}
	Let $\M=(\N, \I)$ be a matroid and $B_1, B_2\in \B$ be two bases. Then there is a bijection $\phi:B_1 \rightarrow B_2$ such that for every $b\in B_1$ we have $B_1\setminus\{b\}\cup\{\phi(b)\}\in\B$.
\end{lemma}
	
Our randomized algorithms aim to find a fractional solution $y$ which maximizes $F(y)$. To round this solution to an integral one, we appeal to the Swap-Rounding procedure, introduced by~\citet{ChekuriVZ10}.
	
\begin{lemma}[Swap-Rounding, \citet{ChekuriVZ10}]\label{lem:rounding}
	Let $(y_1, \ldots, y_n)\in P(\M)$ be a fractional solution in the matroid polytope and $(Y_1, \ldots, Y_n)\in P(\M)\cap \{0,1\}^n$ be an integral solution obtained using randomized swap rounding. Then for any submodular function $f:\{0,1\}^n\rightarrow \mathbb{R}$ and its multilinear extension $F:[0,1]^n\rightarrow \mathbb{R}$, it holds that \[\ex{}{f(Y_1, \ldots, Y_n)}\geq F(y_1, \ldots, y_n).\]
\end{lemma}

\subsection{Differential Privacy}\label{sec:privacy}
Informally, differential privacy is a property that a randomized algorithm satisfies if its output distribution does not change significantly under the change of a single data point, which in our case corresponds to a single agent $J$ and their function $f_J$.
	
More formally, let $D, D'\in\D$ be two sets of agent functions. We say that $D,D'$ are \emph{neighbors}, denoted as $D\sim D'$, if they differ in at most one agent, that is, $D=D'\cup \{f_J\}\setminus\{f_{J'}\}$.
\begin{definition}[Differential Privacy,~\cite{DworkMNS06}]
	A randomized algorithm $\mathcal{A}:D \rightarrow \mathcal{R}$ is  {\em $(\eps,\delta)$-differentially private} if for all neighboring sets $D, D'$ and any measurable output set $R\subseteq \mathcal{R}$, \[\Pr[\mathcal{A}(D) \in R] \leq \exp(\eps) \Pr[\mathcal{A}(D')\in R] + \delta.\]
	Algorithm $\mathcal{A}$ is {\em $(\eps,0)$-differentially private} if it satisfies the definition for $\delta=0$.
\end{definition}
	
A useful property of differential privacy is that it is closed under post-processing.
\begin{lemma}[Post-Processing,~\cite{DworkMNS06}]\label{lem:post-processing}
	Let $\mathcal{A}: \D \rightarrow \mathcal{R}$ be a randomized algorithm that is $(\eps,\delta)$-differentially private. 
	For every (possibly randomized) $f : \mathcal{R} \to \mathcal{R}'$, $f \circ \mathcal{A}$ is $(\eps,\delta)$-differentially private.
\end{lemma}
	
The Exponential Mechanism is a well-known algorithm, which serves as a building block for many differentially private algorithms. The mechanism is used in cases where we need to choose the optimal output with respect to some scoring function on the data set. More formally, let $\mathcal{R}$ denote the range of the outputs and let $q:\D\times \mathcal{R} \rightarrow \mathbb{R}$ be the scoring function which maps the data set - output pairs to utility scores. 
The sensitivity of the scoring function is defined as \[\Delta q = \max\limits_{r\in\mathcal{R}} \max\limits_{D\sim D'} |q(D,r)-q(D',r)|.\]

Next, we define the Exponential Mechanism and present its guarantees.
\begin{lemma}[Exponential Mechanism,~\cite{McSherryT07}]\label{lem:expmech}
	Let input set $D\in\D$, range $\mathcal{R}$, and utility function $q: \D\times \mathcal{R} \rightarrow \mathbb{R}$. 
	The {\em Exponential Mechanism} $\mathcal{O}_{\eps}(q)$ selects and outputs an element $r\in\mathcal{R}$ with probability proportional to 
	$\exp\left(\frac{\eps \cdot q(D,r)}{2\Delta q}\right)$. 
	The Exponential Mechanism is $(\eps,0)$-differentially private and with probability at least $1-\gamma$,
	\[| \max\limits_{r\in\mathcal{R}} q(D,r) - q(D, \mathcal{O}_{\eps}(q))| \leq \frac{2\Delta q}{\eps}\log\left(\frac{|\mathcal{R}|}{\gamma}\right).\]
\end{lemma}
	
\subsection{Concentration inequalities}\label{sec:concentration}
In order to estimate the multilinear extension of $f$ on any input $y$, we sample random sets based on $y$ and query the value of the function on those sets. We will use the following concentration inequalities, to bound the sampling error of our estimate. The next lemma is proven as Theorems 4.4 and 4.5 of~\cite{MitzenmacherU17}.
\begin{lemma}[Chernoff Bounds,~\citep{MitzenmacherU17}]\label{lem:chernoff}
	Let $X_1, \ldots, X_s$ be independent random variables such that for each $j\in[s]$, $X_j\in[0,R]$. 
	Let $X=\frac{1}{s}\sum_{j=1}^s X_j$ and $\mu=\ex{}{X}$. Then
	\begin{align*}
	&\pr{}{X<(1-\zeta)\mu} \leq \exp(-s\zeta^2\mu/2R)\\
	&\pr{}{X>(1+\zeta)\mu} \leq \exp(-s\zeta^2\mu /3R)\\
	&\pr{}{X>(1+\zeta)\mu} \leq \exp(-s\zeta \mu/3R),
	\end{align*}
	where the latter inequality holds for $\zeta\geq 1$ and the first two hold for $\zeta\in(0,1)$.
\end{lemma}
	
\begin{corollary}[Relative and Additive Chernoff Bound]\label{cor:concentration}
	Let $X_1, \ldots, X_s$ be independent random variables such that for each $j\in[s]$, $X_j\in[0,R]$. 
	Let $X=\frac{1}{s}\sum_{j=1}^s X_j$ and $\mu=\ex{}{X}$. Then for $\alpha,\beta\in[0,1]$,
	\begin{align*}
	&\pr{}{X<(1-\alpha)\mu-\beta} \leq \exp(-s\alpha\beta/R)\\
	&\pr{}{X>(1+\alpha)\mu+\beta} \leq \exp(-s\alpha\beta/3R)
	\end{align*}
\end{corollary}
\begin{proof}
	We first prove the first inequality. Note that if $\mu \le \beta$ then the inequality is trivially true so we only need to consider $\mu>\beta$. 
	Let $\ell = \alpha + \frac{\beta}{\mu}$. Notice that $\ell^2 \ge \frac{2\alpha\beta}{\mu}$. Thus, by Lemma~\ref{lem:chernoff},
	\[\pr{}{X < (1-\alpha)\mu-\beta} =\pr{}{X < (1-\ell)\mu} \leq \exp(-s\ell^2 \mu/2R) \le \exp(-s\alpha\beta/R).\]
	Next we prove the second inequality. Again let $\ell = \alpha + \frac{\beta}{\mu}$. If $\ell < 1$ then by Lemma~\ref{lem:chernoff},
	\[\pr{}{X > (1+\alpha)\mu+\beta} \leq \exp(-s \ell^2 \mu/3R) \le \exp(-2s\alpha\beta/3R).\]
	If $\ell \ge 1$ then
	$\pr{}{X > (1+\alpha)\mu+\beta} \leq \exp(-s\ell \mu/3R) \le \exp(-s\beta/3R) \le \exp(-s\alpha\beta/3R).$
	This concludes the proof of the corollary.
\end{proof}
		
\begin{comment}
The next is an extension of Corollary~\ref{cor:concentration} to random variables of any bounded range.
\begin{corollary}[Relative-Additive Chernoff Bound]\label{cor:genconcentration}
    Let $X_1, \ldots, X_s$ be independent random variables such that for each $j\in[s]$, $X_j\in[r_1,r_2]$ for $r_2\geq r_1$. 
    Let $X=\frac{1}{s}\sum_{j=1}^s X_j$ and $\mu=\ex{}{X}$. Then for $\alpha,\beta\in[0,1]$,
    \begin{align*}
    &\pr{}{X<(1-\alpha)\mu -\beta + \alpha r_1} \leq \exp(-s\alpha \beta/(r_2-r_1))\\
    &\pr{}{X>(1+\alpha)\mu+\beta - \alpha r_1} \leq \exp(-s\alpha\beta/3(r_2-r_1)).
    \end{align*}
\end{corollary}
The proof of the latter follows by a straightforward application of Corollary~\ref{cor:concentration} to the shifted random variables $X_j'=X_j-r_1 \in [0,r_2-r_1]$.
\end{comment}

%% file: monotone.tex
\section{Monotone Decomposable Submodular Maximization}\label{sec:monotone}
In this section, we present our differentially private algorithm for the problem of maximizing a $1$-decomposable monotone function $f(S)=\sum_{I\in D} f_I(S)$ under matroid constraints $\M=(\N,\I)$. That is, $f$ consists of the sum of $|D|=m$ non-negative monotone submodular functions $f_I:2^{\N}\rightarrow [0,1]$ and thus $\max_{S}f(S)\leq m$. %We also denote the maximum value of a single element by $M=\max_{u\in\N} f(\{u\})\leq m$. 
Let $OPT=\argmax_{S\in\I}f(S)$ be an optimal solution and $f(\OPT)$ its value.
	
Algorithm~\ref{alg:ContGreedy} is an adaptation of the Continuous Greedy algorithm introduced by~\cite{CalinescuCPV11}. In each round $t\in[T]$, Algorithm~\ref{alg:ContGreedy} greedily constructs a feasible set $B^{(t,r)}\in\I$ and overall takes an $\eta$-step towards that direction, that is, $y^{(t+1,0)}=y^{(t,0)}+\eta \mathbf{1}_{B^{(t,r)}}$. To construct the feasible set $B^{(t,r)}\in\I$, in each $(t,i)$ round, the algorithm picks an element $u^{(t,i)}\in\N^{(t,i)}$ which is feasible with respect to the current set $B^{(t,i-1)}$ and which approximately maximizes the marginal gain from the current fractional solution $y^{(t,i-1)}$.

We denote the true marginal gain of an element $u\in\N$ in round $(t,i)$ by \[w_D^{(t,i)}(u)= F(y^{(t,i-1)} + \eta\mathbf{1}_{u})-F(y^{(t,i-1)}).\]
	
We let $G(x)$ be the estimate of $F(x)$ for any point $x\in [0,1]^n$. To compute $G$, we generate $s$ uniformly random vectors $r^j\gets \U^n$ for $j\in[s]$ in the beginning of the algorithm and set \[G(x) = \frac{1}{s}\sum_{j=1}^s f(\{ u\in\N: r_u^j < x_u \}).\]
Thus, the estimated marginal gain of an element $u\in\N$ in round $(t,i)$ is \[\tw_D^{(t,i)}(u)= G(y^{(t,i-1)} + \eta\mathbf{1}_{u})-G(y^{(t,i-1)}).\]
	
We similarly define $G_I(x) = \frac{1}{s}\sum_{j=1}^s f_I(\{ u\in\N: r_u^j < x_u \})$ for any agent $I$ and write $w_I^{(t,i)}(u)$ for the marginal gain with respect to $F_I$ and $\tw_I^{(t,i)}(u)$ for the marginal gain with respect to $G_I$.
	
To find the approximately maximizing element $u^{(t,i)}$ in each round, we use the Exponential Mechanism, $\mathcal{O}_{\eps_0}(\cdot)$, with the scoring function $\tw_D^{(t,i)}$. Note that the the scoring function does not explicitly take the protected data $\{f_I\}_{I\in D}$ as an input but we denote its dependence on the agent set $D$ as a subscript.
	
\begin{algorithm}[H]
\caption{$\textsc{Private Continuous Greedy}$}\label{alg:ContGreedy}
\begin{algorithmic}[1]
	\State \textbf{Input}: Utility parameters $\eta,\gamma\in(0,1]$, privacy parameters $\eps, \delta\in(0,1]$, and set of agents $D$.
	\State Let $T \gets \lceil\frac{1}{\eta}\rceil$ and $\eps_0\gets 2\log\left(1+\frac{\eps}{4+\log(1/\delta)}\right)$.\; \label{step:eps0}% and $\eps_0\gets \eps/(e-1)(3+\log(1/\delta))$.\; 
	\State Draw $s=6r^2T^4\log(n/\gamma)$ independent random vectors such that $r^j\gets \U^n$ for all $j\in[s]$.\;
	\State $y^{(1,0)} = \mathbf{1}_{\emptyset}$.
	\For{$t=1, \ldots, T$}
	\State $B^{(t,0)}=\emptyset$.\;
	\For{$i=1, \ldots, r$}
		\State Let $\N^{(t,i)}=\{u\in\N\setminus B^{(t,i-1)} : B^{(t,i-1)}\cup \{u\}\in \I\}$.\;
		\If{$\N^{(t,i)}=\emptyset$} let $y^{(t+1,0)}=y^{(t,i-1)}$ and break the loop.\; \EndIf
		\State Define $\tw_D^{(t,i)}(u)= G(y^{(t,i-1)} + \eta\mathbf{1}_{u})-G(y^{(t,i-1)})$ for all $u\in \N^{(t,i)}$.\;
		\State Let $u^{(t,i)}\gets \mathcal{O}_{\eps_{0}}(\tw_D^{(t,i)})$.\;
		\State Let $y^{(t,i)} = y^{(t,i-1)} + \eta \mathbf{1}_{u^{(t,i)}}$.
		\State Let $B^{(t,i)}\gets B^{(t,i-1)}\cup \{u^{(t,i)}\}$.\;
		\EndFor
		\State $y^{(t+1,0)} = y^{(t,r)}$.
		\EndFor 
		\State \Return \textsc{Swap-Rounding}($y^{(T,r)}, \I)$.
\end{algorithmic}
\end{algorithm}
	
The main result of this section is the following.
\begin{theorem}\label{th:main-monotone}
	Let $f:2^{\N}\rightarrow \mathbb{R}_+$, where $|\N|=n$, be a monotone, $1$-decomposable, submodular function and $\M=(\N,\I)$ a matroid of rank $r$. Algorithm~\ref{alg:ContGreedy} with parameters $\eta$ and $\gamma$ is $(\eps,\delta)$-differentially private and returns a set $S\in\I$ such that, with probability $1-\gamma$, \[\ex{}{f(S)}\geq (1-1/e-O(\eta))f(\OPT)-O\left(\frac{r}{\eta\eps}\log\frac{nr}{\eta\gamma}\cdot\log\frac{1}{\delta}\right).\] Algorithm~\ref{alg:ContGreedy} makes $O\left(\frac{nr^3}{\eta^5}\log\frac{n}{\gamma}\right)$ oracle calls.
\end{theorem}
	
We prove the theorem by combining the utility and privacy guarantees of our algorithm, as stated in Theorem~\ref{thm:utility-monotone} and Theorem~\ref{thm:privacy-monotone}, respectively. We remark that Theorem~\ref{thm:utility-monotone} lower bounds the utility of the fractional solution $F(y^{(T,r)})$. Since $y^{(T,r)}=\sum_{t=1}^T\eta \mathbf{1}_{B^{(t,r)}}$, where $B^{(t,r)}\in\I$ for all $t\in[T]$, it follows that $y^{(T,r)}\in \mathcal{P}(\M)$ and Lemma~\ref{lem:rounding} can be applied to yield the final guarantees of the integral solution returned by the swap-rounding process.\footnote{Note that since $1+\eta\geq T\eta\geq 1$, we might need to scale down $y^{(T,r)}$ by $(1+\eta)$ to ensure that $y^{(T,r)}\in \mathcal{P}(\M)$. This would only change the multiplicative factor of the utility guarantee by a constant in the term $O(\eta)$.} We also note that the number of oracle calls corresponds to the number of samples $s$ used to estimate $G$ accurately in all possible $y^{(t,i)}$ intermediate points of the algorithm. %The number of oracle calls made by the Swap-Rounding process is $O(\frac{r^2}{\eta})$ and does not increase the overall bound asymptotically.
	
\subsection{Function Estimation}
We first show that $G$ is a good proxy for $F$ by bounding the sampling error. The number of samples required by the algorithm is determined as such, so that the sampling error is bounded for any possible run (i.e., sequence of picks) of the algorithm. 

Recall that we define one sample from a multilinear extension $F(x)$ as $f(\{u\in\N: v_u<x_u\})$ where $v\gets \U^n$. The following lemma holds for any submodular function $f$ and shows that each sample's expected value is indeed $F(x)$.
\begin{lemma}\label{lem:expectationG}
	For a submodular function $f:2^{\N}\rightarrow \mathbb{R}$ and its multilinear extension $F$, for any point $x\in[0,1]^{\N}$ \[\ex{v\gets \U^n}{f(\{u\in\N: v_u<x_u\})}=F(x).\]
\end{lemma}
\begin{proof}
We prove this result by induction on $n$. For simplicity, we identify the $n$ element set $\N$ with $[n]$. Recall that, by definition,  $F(x) = \sum_{S \subseteq 2^{\N}} f(S) \prod_{i \in S} x_i \prod_{i \not\in S} (1-x_i).$ For $n = 1$, it holds that
\begin{equation*}
	\ex{v\gets \U^n}{f(\{u\in\N: v_u<x_u\})}= x_1 \cdot f(\{1\}) + (1-x_1)\cdot f(\emptyset)
	= \sum_{S \subseteq 2^{\{1\}}} f(S) \prod_{i \in S} x_i \prod_{i \not\in S} (1-x_i).
\end{equation*}
Now suppose that the statement holds for up to $n-1$, in which case we have that
\begin{align*}
	\ex{v\gets \U^{n}}{f(\{u\in\N: v_u<x_u\})}
	 = &\pr{}{v_{n} < x_{n}}\cdot \ex{v\gets \U^{n}}{f(\{n\} \cup \{u\in \N\setminus \{n\}: v_u<x_u\} )} \\
	& + \pr{}{v_{n} \geq x_{n}}\cdot \ex{v\gets \U^{n}}{f(\{u\in \N\setminus\{n\}: v_u<x_u\})}\\
	= &\pr{}{v_{n} < x_{n}} \cdot \sum_{S \subseteq 2^{\N \setminus \{n\}}} f(S \cup \{n\}) \prod_{i \in S} x_i \prod_{i \notin S, i\neq n} (1-x_i) \\
	& + \pr{}{v_{n} \geq x_{n}}\cdot \sum_{S \subseteq 2^{\N \setminus \{n\}}} f(S) \prod_{i \in S} x_i \prod_{i \notin S, i\neq n} (1-x_i) \\
	=& x_{n} \cdot \sum_{S \subseteq 2^{\N \setminus \{n\}}} f(S \cup \{n\}) \prod_{i \in S} x_i \prod_{i \notin S, i\neq n} (1-x_i) \\
	& + (1 - x_{n}) \cdot \sum_{S \subseteq 2^{\N \setminus \{n\}}} f(S) \prod_{i \in S} x_i \prod_{i \notin S, i\neq n} (1-x_i) \\
	=&\sum_{x_{n} \in S \subseteq 2^{\N}} f(S) \prod_{i \in S} x_i \prod_{i \notin S, i\neq n} (1-x_i)\\
	 & + \sum_{S \subseteq 2^{\N \setminus \{n\}}} f(S) \prod_{i \in S} x_i \prod_{i \notin S} (1-x_i) \\
	=& \sum_{S \subseteq 2^{\N}} f(S) \prod_{i \in S} x_i \prod_{i \notin S} (1-x_i) =F(x).
\end{align*}
This concludes the proof of the lemma.
\end{proof}

The next lemma shows that the estimated marginal gains $\tw_D^{(t,i)}(u)$ are indeed concentrated around their expected value $w_D^{(t,i)}(u)$.
\begin{lemma}\label{lem:gainsamplingerror}
With probability at least $1-2\gamma$, for any sequence of points picked by the algorithm $\left\{\{u^{(t,i)}\}_{i=1}^r\right\}_{t=1}^T$ and any $u\in\N$, it holds that
\[(1-\eta)w_D^{(t,i)}(u)-\frac{\eta f(\OPT)}{rT}\leq \tw_D^{(t,i)}(u)\leq (1+\eta)w_D^{(t,i)}(u)+\frac{\eta f(\OPT)}{rT}.\]
\end{lemma}
\begin{proof}
Fix a point $y^{(t,i-1)}$ and element $u\in\N$. Then $\tw_D^{(t,i)}=G(y^{(t,i-1)}+\eta\mathbf{1}_u)-G(y^{(t,i-1)})=\frac{1}{s}\sum_{j=1}^s X_j$ where we define
\[X_j=f(\{v\in\N: r^j_v<y^{(t,i-1)}_v+\eta \mathbbm{1}\{v=u\}\})-f(\{v\in\N: r^j_v<y^{(t,i-1)}_v\}).\]
The random variables $X_j$ are independent since the random vectors $r^j$ are also independent. Let us denote the random set based on the vector $r^j$ as $R^j=\{v\in\N: r^j_v<y^{(t,i-1)}_v\}$. Then either $X_j=0$ or $X_j=f(R^j\cup\{u\})-f(R^j)$. In the latter case, by submodularity, $X_j\leq f(\{u\})-f(\emptyset)$ and by non-negativity $X_j\leq f(\OPT)$. By monotonocity, $X_j\geq 0$. Thus, it always holds that $X_j\in[0,f(\OPT)]$.
		
By Lemma~\ref{lem:expectationG} and linearity of expectation, it also holds that 
\begin{equation}\label{eq:expgain}
\ex{}{\tw_D^{(t,i)}(u)}=\frac{1}{s}\sum_{j=1}^s\ex{}{X_j}=F(y^{(t,i-1)}+\eta\mathbf{1}_u)-F(y^{(t,i-1)})=w_D^{(t,i)}(u) 		\end{equation}
		
Note that the number of samples $s=6r^2T^4\log(n/\gamma)\geq \frac{3rT}{\eta^2}\log(n\cdot n^{rT}/\gamma)$.
By Corollary~\ref{cor:concentration} and equation~\eqref{eq:expgain} and since $X_j\in[0,f(\OPT)]$, we have that:
\begin{align*}
		&\pr{}{\tw_D^{(t,i)}(u)<(1-\eta)w_D^{(t,i)}(u)-\frac{\eta f(\OPT)}{rT}} \leq \exp\left(\frac{-s\eta \frac{\eta f(\OPT)}{rT}}{f(\OPT)}\right)=\exp\left(-3rT\log\frac{n}{\gamma}\right)\leq \frac{\gamma}{n^{rT+1}}\\
		&\pr{}{\tw_D^{(t,i)}(u)>(1+\eta)w_D^{(t,i)}(u)+\frac{\eta f(\OPT)}{rT}} \leq \exp\left(\frac{-s\eta \frac{\eta f(\OPT)}{rT}}{3f(\OPT)}\right)=\exp\left(-rT\log\frac{n}{\gamma}\right)\leq \frac{\gamma}{n^{rT+1}}
	\end{align*}
	It follows by union bound that with probability $1-2\gamma/n^{rT+1}$, \[(1-\eta)w_D^{(t,i)}(u)-\frac{\eta f(\OPT)}{rT}\leq \tw_D^{(t,i)}(u)\leq (1+\eta)w_D^{(t,i)}(u)+\frac{\eta f(\OPT)}{rT}.\]
	The number of all possible sequences of points $y^{(t,i)}$ is at most the number of possible sequences of picked points $u^{(t,i)}$, which is bounded by $n^{rT}$. By union bound over all sequences of points $y^{(t,i)}$ and elements $u\in\N$, the inequalities hold with probability at least $1-2\gamma$.
\end{proof}
	
\subsection{Utility Analysis}
Having bounded the error introduced by sampling, we now move on to bounding the error introduced by privacy. 
The next claim (Claim~\ref{claim:sensitivity}) bounds the sensitivity of the estimated marginal gains, which in turn bounds the error of the Exponential Mechanism (Claim~\ref{claim:expmecherror}). Note that, for ease of exposition, we consider neighboring datasets to be $A, B$ such that $A\setminus B=I$, which we later make up for in the privacy analysis of this section.
\begin{claim}\label{claim:sensitivity}
	Let us denote the sensitivity of $\tw^{(t,i)}_D$ by $\Delta\tw^{(t,i)}$. %$=\max_{A\sim B} \max_{u\in\N} |\tw^{(t,i)}_A(u)-\tw^{(t,i)}_B(u)|$.
	For any $t\in[T], i\in[r]$, it holds that $\Delta\tw^{(t,i)}\leq 1$.
\end{claim}
\begin{proof}
	Let $A,B$ be any two neighboring datasets such that $A\setminus B=I$. Then
	\begin{align*}
		\Delta\tw^{(t,i)}& = \max_{u\in \N} |\tw^{(t,i)}_A(u)  - \tw^{(t,i)}_B(u)|\\
		& = \max_{u\in \N} |G_A(y^{(t,i-1)}+\eta\mathbf{1}_{u})-G_A(y^{(t,i-1)})- G_B(y^{(t,i-1)}+\eta\mathbf{1}_{u})+G_B(y^{(t,i-1)})|\\
		& = \max_{u\in \N} |G_I(y^{(t,i-1)}+\eta\mathbf{1}_{u})-G_I(y^{(t,i-1)})|\\
		& = \max_{u\in \N} \left|\frac{1}{s}\sum_{j=1}^s \left[f_I(\{v\in\N : r_v^j<y^{(t,i-1)}_v+\eta\mathbbm{1}\{v=u\}\}) - f_I(\{v\in\N : r_v^j<y^{(t,i-1)}_v\})\right]\right|\\
		&\leq \max_{u\in \N} \frac{1}{s}\sum_{j=1}^s |f_I(\{v\in\N : r_v^j<y^{(t,i-1)}_v\}\cup\{u\}) - f_I(\{v\in\N : r_v^j<y^{(t,i-1)}_v\})|\\
		%& \leq \max_{u\in N} \frac{1}{s}\sum_{j=1}^s\max\left\{f_I(\N\setminus\{u\})-f_I(\N), f_I(\{u\}) - f_I(\emptyset)\right\} \tag{by submodularity}\\
		%& \leq \max_{u\in\N} \max\left\{f_I(\N\setminus\{u\}),f_I(\{u\})\right\} \tag{by non-negativity}\\
		%& \leq 1 \tag{since $f$ is $1$-decomposable.}
		&\leq 1. \tag{since $f_I:2^{\N}\rightarrow [0,1]$ $\forall I\in D$}
	\end{align*}
\end{proof}
	
\begin{claim}\label{claim:expmecherror}
	With probability at least $1-\gamma$, for all $t\in[T],i\in[r]$, and for all $u\in\N$,
	\begin{equation*}
		\tw_D^{(t,i)}(u^{(t,i)})\geq \tw_D^{(t,i)}(u)-\frac{2}{\eps_0}\log(nrT/\gamma).
	\end{equation*}
\end{claim}
\begin{proof}
	Fix round $(t,i)$. By the guarantees of the Exponential Mechanism (Lemma~\ref{lem:expmech}), with probability $1-\frac{\gamma}{rT}$ we have that for all $u\in\N$,
	\begin{align*}
		\tw_D^{(t,i)}(u^{(t,i)})&\geq \tw_D^{(t,i)}(u)-\frac{2\Delta\tw^{(t,i)}}{\eps_0}\log\frac{|\N^{(t,i)}|rT}{\gamma}\\
		& \geq \tw_D^{(t,i)}(u)-\frac{2}{\eps_0}\log\frac{nrT}{\gamma} \tag{by Claim~\ref{claim:sensitivity} and $|\N^{(t,i)}|\leq |\N|=n$}
	\end{align*}
	
By union bound, the inequality holds for all rounds $i\in[r]$ and $t\in[T]$ and $u\in\N$ with probability at least $1-\gamma$.
\end{proof}
	
We are now ready to prove the utility guarantee of our algorithm. The proof follows the main steps of the proof of the performance guarantees of the Continuous Greedy algorithm, yet accounting for the discretization, the sampling error, and the error of the Exponential Mechanism.
\begin{theorem}\label{thm:utility-monotone}
With probability at least $1-3\gamma$, \[F(y^{(T+1,0)}) \geq (1-1/e-O(\eta))f(\OPT)-\frac{8r}{\eta\eps_0}\log\frac{nr}{\eta\gamma}.\]
\end{theorem}
\begin{proof}
We condition on the event that for all rounds $t\in[T]$, $i\in[r]$, and elements $u\in\N$, \[(1-\eta)w_D^{(t,i)}(u)-\frac{\eta f(\OPT)}{rT}\leq \tw_D^{(t,i)}(u)\leq (1+\eta)w_D^{(t,i)}(u)+\frac{\eta f(\OPT)}{rT}.\]
By Lemma~\ref{lem:gainsamplingerror}, this is true with probability at least $1-2\gamma$. We start by lower bounding the increase in utility between the fractional solutions at the beginning of two consecutive outer loops. It holds:
\begin{align*}
F(y^{(t+1,0)})-F(y^{(t,0)}) & = F(y^{(t,r)})-F(y^{(t,0)}) \\
& =\sum_{i=1}^r [F(y^{(t,i)})-F(y^{(t,i-1)})] \\
& =\sum_{i=1}^r w^{(t,i)}_D(u^{(t,i)}) \tag{since $y^{(t,i)} = y^{(t,i-1)}+\eta\mathbf{1}_{u^{(t,i)}}$}\\
& \geq \sum_{i=1}^r \frac{1}{1+\eta}\tw^{(t,i)}_D(u^{(t,i)}) - r\frac{\eta f(\OPT)}{(1+\eta) r T} \tag{by Lemma~\ref{lem:gainsamplingerror}}\\
& \geq \frac{1}{1+\eta}\sum_{i=1}^r \tw^{(t,i)}_D(u^{(t,i)}) - \frac{\eta f(\OPT)}{T}
\end{align*}
		
We assume that $|B^{(t,r)}|=r$ for any $t\in[T]$. This can be achieved without loss of generality by adding ``dummy'' elements of value $0$. By Lemma~\ref{lem:mappingtoOPT}, there exists a bijection $\phi$ such that $\phi(u^{(t,i)})=o^{(t,i)}$, where $\OPT=\{o^{(t,1)}, \ldots, o^{(t,r)}\}$ is the optimal solution. Now, note that since $o^{(t,i)}$ is a feasible option in the $i$-th round, by the guarantees of the Exponential Mechanism (Claim~\ref{claim:expmecherror}), with probability $1-\gamma$ we have that
\begin{equation*}
\tw_D^{(t,i)}(u^{(t,i)})\geq \tw_D^{(t,i)}(o^{(t,i)})-\frac{2}{\eps_0}\log\frac{nrT}{\gamma},
\end{equation*}
for all rounds $i\in[r]$ and $t\in[T]$. We condition on this event for the rest of the proof.
	
It follows that with probability $1-3\gamma$,
\begin{align}\label{eq:Ftdiff}
F(y^{(t+1,0)})-F(y^{(t,0)}) & \geq \frac{1}{1+\eta}\sum_{i=1}^r \tw_D^{(t,i)}(o^{(t,i)}) - \frac{\eta f(\OPT)}{T} -\frac{2r}{(1+\eta)\eps_0}\log\frac{nrT}{\gamma} \nonumber \\
& \geq \frac{1-\eta}{1+\eta}\sum_{i=1}^r w_D^{(t,i)}(o^{(t,i)}) - \frac{2\eta f(\OPT)}{ T}-\frac{2r}{\eps_0}\log\frac{nrT}{\gamma},
\end{align}
where the last inequality follows from Lemma~\ref{lem:gainsamplingerror}. 
\begin{claim}\label{claim:sumweightsOPT}
For all $t\in[T]$, 
\[\sum_{i=1}^r w_D^{(t,i)}(o^{(t,i)})\geq \eta [F(y^{(t,r)}\vee\mathbf{1}_{\OPT})-F(y^{(t,r)})].\]
Moreover, since $f$ is monotone, $F(y^{(t,r)}\vee\mathbf{1}_{\OPT})\geq f(\OPT)$.
\end{claim}
\begin{proof}[Proof of Claim~\ref{claim:sumweightsOPT}]
By the definition of the weights, it holds that
\begin{align*}
\sum_{i=1}^r w_D^{(t,i)}(o^{(t,i)}) & = \sum_{i=1}^r F(y^{(t,i-1)}+\eta\mathbf{1}_{o^{(t,i)}})-F(y^{(t,i)}) \\
& = \eta \sum_{i=1}^r [F(y^{(t,i-1)}\vee \mathbf{1}_{o^{(t,i)}}) - F(y^{(t,i-1)}\wedge \mathbf{1}_{\bar{o}^{(t,i)}})]\\
& \geq \eta \sum_{i=1}^r [F(y^{(t,i-1)}\vee 1_{o^{(t,i)}}) - F(y^{(t,i-1)})] \tag{by monotonicity} \\
& \geq \eta \sum_{i=1}^r [F(y^{(t,r)}\vee 1_{o^{(t,i)}}) - F(y^{(t,r)})] \\
&\geq \eta [F(y^{(t,r)}\vee 1_{\OPT}) - F(y^{(t,r)})], \tag{by submodularity}
\end{align*}
where the second to last inequality follows by monotonicity and submodularity, since $y^{(t,r)} \geq y^{(t,i-1)}$ for all $i\in[r]$.
\end{proof}
	
Substituting the bounds of Claim~\ref{claim:sumweightsOPT} in inequality~\eqref{eq:Ftdiff}, we get that 
\begin{align*}
F(y^{(t+1,0)})-F(y^{(t,0)}) & \geq \eta\frac{1-\eta}{1+\eta}[f(\OPT)-F(y^{(t+1,0)})] - \frac{2\eta f(\OPT)}{T}-\frac{2r}{\eps_0}\log\frac{nrT}{\gamma}\\
& \geq \eta[(1-2\eta)f(\OPT)-F(y^{(t+1,0)})]-\frac{2\eta f(\OPT)}{ T}-\frac{2r}{\eps_0}\log\frac{nrT}{\gamma},
\end{align*}
where the last inequality holds by non-negativity.
	
Let us denote $\Omega = (1-2\eta)f(\OPT)$ and $\xi=\frac{2\eta f(\OPT)}{ T}+\frac{2r}{\eps_0}\log\frac{nrT}{\gamma}$. Then,
\begin{align*}
& F(y^{(t+1,0)})-F(y^{(t,0)}) \geq \eta[\Omega-F(y^{(t+1,0)})]-\xi\\
&\Rightarrow \Omega-F(y^{(t+1,0)}) \leq \Omega-F(y^{(t,0)})-\eta[\Omega-F(y^{(t+1,0)})]+\xi\\
& \Rightarrow \Omega-F(y^{(t+1,0)}) \leq \frac{\Omega-F(y^{(t,0)})}{1+\eta} + \frac{\xi}{1+\eta}\\
& \Rightarrow \Omega-F(y^{(T+1,0)}) \leq \frac{\Omega-F(y^{(1,0)})}{(1+\eta)^T} + \xi\sum_{t=1}^T (1+\eta)^{-t}\\
& \Rightarrow \Omega-F(y^{(T+1,0)}) \leq \frac{1}{(1+\eta)^T}\Omega + T\xi \tag{by non-negativity $F(y^{(1,0)})=f(\emptyset)\geq0$}\\
& \Rightarrow F(y^{(T+1,0)}) \geq \left(1-\frac{1}{(1+\eta)^T}\right)\Omega - T\xi.
\end{align*}
By substituting the values for $\Omega$ and $\xi$, we have that with probability at least $1-3\gamma$,
\begin{align*}
F(y^{(T+1,0)})
&\geq (1-1/e -O(\eta))f(\OPT) - 2\eta f(\OPT)-\frac{2rT}{\eps_0}\log\frac{nrT}{\gamma} \\
&\geq (1-1/e -O(\eta))f(\OPT)-\frac{8r}{\eta\eps_0}\log\frac{nr}{\eta\gamma} \tag{since $T=\lceil\frac{1}{\eta}\rceil\leq \frac{2}{\eta}$}
\end{align*}
This concludes the proof of the theorem.
\end{proof}
	
\subsection{Privacy Analysis}
For the privacy analysis, we need the following concentration bound (Claim~\ref{claim:privacyconc}). A stronger version of this bound also appears in~\citep{GuptaLMRT10}, but its proof is not entirely correct. We provide a correct proof, albeit for a slightly weaker statement with respect to the constant factor.

Consider a probabilistic process that proceeds in rounds, in which a random variable is picked from different distributions in each round. Informally, the next claim bounds the sum of the expected values of the samples by the sum of their realized values.
\begin{claim}\label{claim:privacyconc}
Consider an $n$-round probabilistic process. In each round $i\in[n]$, an adversary chooses a distribution $\D_i$ over $[0,1]$ and a sample $R_i$ is drawn from this distribution. Let $Z_1 = 1$ and $Z_{i+1} = Z_i - R_i Z_i$. We define the random variable $Y_j=\sum_{i=j}^n Z_i \ex{}{R_i}$. Then for any $j\in[n]$, \[ \pr{}{Y_j \geq q Z_j} \leq \exp (3 - q).\] %In particular, this implies that $\pr{}{Y_1 \geq q } \leq \exp (3 - q).$
\end{claim}
\begin{proof}
The proof is by reverse induction on $j$. For $j=n$, $Y_n = \ex{}{R_n} Z_n \leq Z_n$ since $\D_n$ is a distribution on $[0,1]$ and has expectation at most $1$. It follows that $\pr{}{Y_n \geq q Z_n}=0$ for any $q>1$ and the claim is trivially true for $j=n$ and for any $q$.
			
For the inductive step, suppose that $\pr{}{Y_{j+1}\geq q Z_{j+1}} \leq \exp(3-q)$. We will prove that $\pr{}{Y_j \geq q Z_j} \leq \exp(3-q)$. For $q \leq 3$ the LHS is at least $1$, so the claim is trivially true. Let us denote $\mu_j=\ex{}{R_j}$. In what follows, the expectation is over $R_j\gets \D_j$. It holds that
\begin{align*}
\pr{}{Y_j \geq q Z_j} &= \ex{}{ \pr{}{Y_{j+1} \geq q Z_j - \mu_j Z_j} } \\
&= \ex{}{ \pr{}{Y_{j+1} \geq \frac{q-\mu_j}{1 - R_j} \cdot Z_{j+1} } } \\
&\leq \ex{}{ \exp \left(3 - \frac{q-\mu_j}{1 - R_j}\right)} \tag{by the inductive hypothesis}
\end{align*}
So it suffices to prove that $\ex{}{\exp \left(3 - \frac{q-\mu_j}{1 - R_j}\right)  } \leq \exp(3 - q)$ for the case $q>3$.
This is equivalent to proving the following inequality, for $q>3$:
\begin{equation*}
\ex{}{\exp \left( \frac{\mu_j - q R_j}{1 - R_j} \right)} \leq 1
\end{equation*}		
Let us denote $f(R_j)=\exp \left( \frac{\mu_j - q R_j}{1 - R_j} \right)$. Calculating the derivatives of the function $f$ with respect to $R_j$, we get that:
 \begin{align*}
 &f'(R_j)=\exp \left( \frac{\mu_j - q R_j}{1 - R_j} \right)\frac{\mu_j-q}{(1-R_j)^2} \intertext{and}
 &f''(R_j)=\exp \left( \frac{\mu_j - q R_j}{1 - R_j} \right)\frac{\mu_j-q}{(1-R_j)^4}(\mu_j-q+2-2R_j)
 \end{align*}
Since $q>3$ and $R_j,\mu_j\in[0,1]$, it holds that $f''(R_j)>0$ and the function $f$ is convex with respect to $R_j$.
Therefore, \[\ex{}{f(R_j)}\leq \ex{}{(1-R_j) f(0)+R_j f(1)}=(1-\mu_j) f(0)+\mu_j f(1)=(1-\mu_j)\exp(\mu_j)+0 \leq 1.\]
This concludes the proof of the inductive step and the proof of the claim.
\end{proof}
	
We now prove the privacy guarantees of Algorithm~\ref{alg:ContGreedy}. The proof follows the same steps as the privacy proof for the CPP problem of~\cite{GuptaLMRT10}. The main idea is to bound the privacy loss by a function of the sum of the expected marginal gains of a player $I$, where each expected value is taken over different (possibly adversarial) distributions. Using the concentration of Claim~\ref{claim:privacyconc}, we can bound this sum by the sum of the \emph{realized} marginal gains of player $I$. Intuitively, the total ``movement'' of a monotone function on a sequence of increasing points has to be bounded by its range. So, since $f_I$ is monotone and has range in $[0,1]$, this sum of realized marginal gains must be bounded by $1$.
\begin{theorem}\label{thm:privacy-monotone}
Algorithm~\ref{alg:ContGreedy} is $\left((e^{\eps_0/2}-1)(4+\log \frac{1}{\delta}), \delta \right)$-differentially private.
\end{theorem}
\begin{proof}
Let $A$ and $B$ be two sets of agents such that $A \triangle B = \{I\}$. Suppose that instead of the output set, we reveal the sequence in which we pick the elements of our algorithm and let this sequence be denoted as $U = (u^{(1,1)}, u^{(1,2)}, \dots, u^{(T,r)})$. We are then interested in bounding the ratio of the probabilities that the output sequence be $U$ under input $A$ and $B$. By the post-processing property (Lemma~\ref{lem:post-processing}), this suffices to achieve the same privacy parameters over the output of the algorithm, \textsc{Swap-Rounding}$(y^{(T,r)}, \I)$.\footnote{Note that swap-rounding does not access the function again.}
	
We recall that the scores used in the Exponential Mechanism are $\tw_D^{(t,i)}(u) = G(y^{(t,i-1)} + \eta\mathbf{1}_{u})-G(y^{(t,i-1)})$. For ease of notation, we drop the irrelevant parameters of our algorithm and denote it by $M$. By the chain rule of probability,	
\begin{align}\label{eq:privacyratio}
	\frac{\pr{}{\textsc{M}(A) = U}}{\pr{}{\textsc{M}(B) = U}}
	&= \prod_{t=1}^T \prod_{i=1}^r \frac{\exp(\frac{\eps_0}{2} \tw_A^{(t,i)}(u^{(t,i)}))/\sum_{u\in\N^{(t,i)}} \exp(\frac{\eps_0}{2} \tw_A^{(t,i)}(u))}{\exp(\frac{\eps_0}{2} \tw_B^{(t,i)}(u^{(t,i)}))/\sum_{u\in\N^{(t,i)}} \exp(\frac{\eps_0}{2} \tw_B^{(t,i)}(u))} \nonumber\\
	&= \left( \prod_{t=1}^T\prod_{i=1}^r \frac{\exp(\frac{\eps_0}{2} \tw_A^{(t,i)}(u^{(t,i)}))}{\exp(\frac{\eps_0}{2} \tw_B^{(t,i)}(u^{(t,i)}))}\right) \left( \prod_{t=1}^T\prod_{i=1}^r \frac{\sum_{u\in\N^{(t,i)}} \exp(\frac{\eps_0}{2} \tw_B^{(t,i)}(u))}{\sum_{u\in\N^{(t,i)}} \exp(\frac{\eps_0}{2} \tw_A^{(t,i)}(u))}\right)
\end{align}
We divide this analysis into two cases: $A \setminus B = \{I\}$ and $B \setminus A = \{I\}$. 
		
If $A \setminus B = \{I\}$, the second factor of~\eqref{eq:privacyratio} is bounded above by $1$, since $\tw_I^{(t,i)}(u)\geq0, \forall u\in\N$, and the first factor is bounded as follows.
\begin{align*}
	\prod_{t=1}^T\prod_{i=1}^r \frac{\exp(\frac{\eps_0}{2} \tw_A^{(t,i)}(u^{(t,i)}))}{\exp(\frac{\eps_0}{2} \tw_B^{(t,i)}(u^{(t,i)}))}
	&= \prod_{t=1}^T\prod_{i=1}^r \frac{\exp(\frac{\eps_0}{2} \tw_{I}^{(t,i)}(u^{(t,i)}) + \frac{\eps_0}{2}\tw_B^{(t,i)}(u^{(t,i)}))}{\exp(\frac{\eps_0}{2} w_B^{(t,i)}(u^{(t,i)}))} \\
	&=\prod_{t=1}^T\prod_{i=1}^r \exp(\frac{\eps_0}{2} \tw_{I}^{(t,i)}(u^{(t,i)}) )\\
	&=\exp(\sum_{t=1}^T\sum_{i=1}^r \frac{\eps_0}{2} \tw_{I}^{(t,i)}(u^{(t,i)}) ) \\
	&= \exp( \frac{\eps_0}{2} \sum_{t=1}^T\sum_{i=1}^r  G_I(y^{(t,i-1)} + \eta\mathbf{1}_{u^{(t,i)}}) - G_I(y^{(t,i-1)}) ) \\
	&= \exp(\frac{\eps_0}{2} \left(G_I (y^{(T,r)}) - G_I(y^{(1,0)})\right) )\\
	&\leq \exp(\frac{\eps_0}{2}),
\end{align*}
where the last inequality holds since $f_I(S)\in[0,1]$ for any $S\subseteq \N$.
		
Let us now consider the second case. If $B \setminus A = \{I\}$, the first factor of~\eqref{eq:privacyratio} is bounded from above by $1$, and the second factor is bounded as follows.
\begin{align*}
	&\prod_{t=1}^T \prod_{i=1}^r \frac{\sum_{u \in \N^{(t,i)}} \exp(\frac{\eps_0}{2} \tw_B^{(t,i)}(u))}{\sum_{u \in \N^{(t,i)}} \exp(\frac{\eps_0}{2} \tw_A^{(t,i)} (u))} \\
	= &\prod_{t=1}^T \prod_{i=1}^r \frac{\sum_{u \in \N^{(t,i)}} \exp(\frac{\eps_0}{2} \tw_I^{(t,i)}(u))\exp(\frac{\eps_0}{2} \tw_A^{(t,i)}(u))}{\sum_{u \in \N^{(t,i)}} \exp(\frac{\eps_0}{2} \tw_A^{(t,i)}(u))} \\
	= &\prod_{t=1}^T \prod_{i=1}^r  \ex{ u\gets P^{(t,i)}}{\exp(\frac{\eps_0}{2} \tw_I^{(t,i)}(u) )} \\
	= &\prod_{t=1}^T \prod_{i=1}^r  \ex{ u\gets P^{(t,i)}}{\exp(\frac{\eps_0}{2} (G_I(y^{(t,i-1)} + \eta\mathbf{1}_{u}) - G_I(y^{(t,i-1)}) ) )}\\
	\leq &\prod_{t=1}^T \prod_{i=1}^r  \ex{ u\gets P^{(t,i)}}{1+(e^{\eps_0/2}-1)(G_I(y^{(t,i-1)} + \eta\mathbf{1}_{u}) - G_I(y^{(t,i-1)}) ) } \tag{$e^x \leq 1 + \frac{e^{\eps_0/2}-1}{\eps_0/2} x ~ \forall x\in[0,\frac{\eps_0}{2}]$}\\
	\leq &\exp( (e^{\eps_0/2}-1)\sum_{t=1}^T \sum_{i=1}^r \ex{u\gets P^{(t,i)}}{G_I(y^{(t,i-1)} + \eta\mathbf{1}_{u}) - G_I(y^{(t,i-1)})}) \tag{$1 + t \leq e^t ~ \forall t$}.
\end{align*}
 In this calculation, we define the distributions $P^{(t,i)}$ on $\N$ to be $P^{(t,i)}(u)= \frac{\exp(\frac{\eps_0}{2} \tw_A^{(t,i)}(u))}{\sum_{u \in \N^{(t,i)}} \exp(\frac{\eps_0}{2} \tw_A^{(t,i)}(u))}$ for $u\in \N^{(t,i)}$ and $P^{(t,i)}(u)=0$ otherwise.

Next, we will bound the sum \[\sum_{t=1}^T \sum_{i=1}^r \ex{u\gets P^{(t,i)}}{G_I(y^{(t,i-1)} + \eta\mathbf{1}_{u}) - G_I(y^{(t,i-1)})}.\]
To this end, we will use Claim~\ref{claim:privacyconc}. Consider a $Tr$-round probabilistic process. In each round $(t,i)$, an adversary chooses a distribution $P^{(t,i)}$ and then a sample $u^{(t,i)}$ is drawn from that distribution. Let $Z_{(t,i)}=1-G_I(y^{(t,i-1)})$ be the total remaining realized utility at the beginning of round $(t,i)$. Define the random variable \[R_{(t,i)}(u)=\frac{G_I(y^{(t,i-1)} +\eta\mathbf{1}_{u})-G_I(y^{(t,i-1)})}{1-G_I(y^{(t,i-1)})},\] which denotes the percentage increase in utility if element $u$ is picked during round $(t,i)$. The distribution of $R_{(t,i)}(u)$, denoted by $D^{(t,i)}$, is directly determined by $P^{(t,i)}$. Also note that, since $f_I$ is monotone and has range in $[0,1]$, $R_{(t,i)}(u)\in[0,1]$. After the element of the round $u^{(t,i)}$ has been chosen, the total remaining realized utility is updated as $Z_{(t,i)}=Z_{(t,i-1)}-R_{(t,i-1)}(u^{(t,i)})Z_{(t,i-1)}$. Finally, let $Y_{(\tau,j)}=\sum_{(t,i)\geq(\tau,j)}\ex{u\gets P^{(t,i)}}{R_{(t,i)}(u)}Z_{(t,i)}$. By these definitions, the sum we want to bound is exactly $Y_{(1,1)}$. 
By Claim~\ref{claim:privacyconc}, the sum $Y_{(1,1)}$ as defined through this random process satisfies
\begin{align*}
	&\pr{}{Y_{(1,1)}\geq qZ_{(1,1)}}\leq \exp(3-q)\\
	\Leftrightarrow &\pr{}{Y_{(1,1)}\geq q(1-G_I(y^{(1,0)}))}\leq \exp(3-q)\\
	\Rightarrow &\pr{}{Y_{(1,1)}\geq q}\leq \exp(3-q)\\
	\Rightarrow &\pr{}{\sum_{t=1}^T \sum_{i=1}^r \ex{u}{G_I (y^{(t,i)} + \eta 1_{u}) - G(y^{(t,i)}))} \geq 3+\log \frac{1}{\delta}} \leq \delta
\end{align*}
		
We conclude that, for $B \setminus A = \{I\}$, with probability at least $1-\delta$, the ratio of equation~\eqref{eq:privacyratio} is bounded by
\[\frac{\pr{}{\textsc{M}(A) = U}}{\pr{}{\textsc{M}(B) = U}} \leq \exp((e^{\eps_0/2}-1)(3+\log \frac{1}{\delta})).\]

More generally, for any two neighboring sets of agents $A\sim B$, with probability $1-\delta$, the ratio is bounded by 
\begin{align*}
\frac{\pr{}{\textsc{M}(A) = U}}{\pr{}{\textsc{M}(B) = U}} 
 = \frac{\pr{}{\textsc{M}(A) = U}}{\pr{}{\textsc{M}(A\cap B) = U}}\cdot \frac{\pr{}{\textsc{M}(A\cap B) = U}}{\pr{}{\textsc{M}(B) = U}}
 \leq \exp(\frac{\eps_0}{2}+(e^{\eps_0/2}-1)(3+\log \frac{1}{\delta})).
\end{align*}

It follows that Algorithm~\ref{alg:ContGreedy} is $\left((e^{\eps_0/2}-1)(4+\log \frac{1}{\delta}), \delta \right)$-differentially private.
\end{proof}

%% file: nonmonotone.tex
\section{Non-monotone Decomposable Submodular Maximization}\label{sec:nonmonotone}
In this section, we present our algorithm for the case of non-monotone functions. Algorithm~\ref{alg:MeasContGreedy} is an adaptation of the Measured Continuous Greedy algorithm introduced by~\citet{FeldmanNS11}. The main difference from Algorithm~\ref{alg:ContGreedy} is the update step in line~\ref{step:update-nonmonotone}, which also leads to the change of the definition of the marginal gains $\tw_D^{(t,i)}(u)$ in line~\ref{step:weights-nonmonotone}, so that $\tw_D^{(t,i)}(u^{(t,i)})=G(y^{(t,i)})-G(y^{(t,i-1)})$ still holds.

A technical difference with the previous section is that, since $f$ is not monotone, the marginal gain of an element could be negative. In practice, non-private algorithms for the problem skip any round in which all available elements have negative marginal gain. To simulate this behavior and facilitate the utility analysis, it is common to assume instead that there exists a set $\N'$ of $r$ elements in $\N$, whose marginal contribution to any set is $0$, that is $\forall S\subseteq\N$, $f(S)=f(S\setminus \N')$. If this condition is not true, then we augment the set of elements $\N=\N\cup\N'$. 

Since both skipping a round as well as checking for this condition would be problematic for privacy, we choose to always augment the ground set with $r$ ``dummy'' elements in the beginning. So for the rest of this section, we will assume that $\N\gets \N\cup\N'$ and $n\gets n+r$. This establishes that in every round there exists at least one element with non-negative marginal gain. Any dummy elements are later removed from the output set of the algorithm. 
This does not affect the privacy guarantees (due to the post-processing property), the optimal solution, or the value of the output\footnote{The size of the ground set $n$ at most doubles, which only affects the utility by a constant factor in the $\log$.}. 

\begin{algorithm}[H]
\caption{$\textsc{Private Measured Continuous Greedy}$}\label{alg:MeasContGreedy}
\begin{algorithmic}[1]
	\State \textbf{Input}: Utility parameters $\eta,\gamma\in(0,1]$, privacy parameters $\eps, \delta\in(0,1]$, and set of agents $D$.\;
	%\State Let $\N=\N\cup\N'$, where $\N'$ is a ``dummy'' copy of $\N$ whose elements have zero value.\;\label{step:dummy}
	\State Let $T \gets \lceil\frac{1}{\eta}\rceil$ and $\eps_0\gets \eps/(14+4\log(1/\delta))$.\;% and $\eps_0\gets 2\left(1-\frac{1}{\exp\left(\frac{\eps}{14+4\log(1/\delta)}\right)}\right)$.\;
	\State Draw $s=48 r^3 T^7 \log (n/\gamma)$ independent random vectors such that $r^j\gets \U^n$ for all $j\in[s]$.\;
	\State $y^{(1,0)} = \mathbf{1}_{\emptyset}$.
	\For{$t=1, \ldots, T$}
	\State $B^{(t,0)}=\emptyset$.\;
	\For{$i=1, \ldots, r$}
	\State Let $\N^{(t,i)}=\{u\in\N\setminus B^{(t,i-1)} : B^{(t,i-1)}\cup \{u\}\in \I\}$.\;
	\If{$\N^{(t,i)}=\emptyset$} let $y^{(t+1,0)}=y^{(t,i-1)}$ and break the loop.\; \EndIf
	\State Define $\tw_D^{(t,i)}(u)= G(y^{(t,i-1)} + \eta (1-y^{(t,i-1)}_u)\mathbf{1}_{u})-G(y^{(t,i-1)})$ for all $u\in \N^{(t,i)}$.\;\label{step:weights-nonmonotone}
	\State Let $u^{(t,i)}\gets \mathcal{O}_{\eps_{0}}(\tw_D^{(t,i)})$.\;
	\State Let $y^{(t,i)} = y^{(t,i-1)} + \eta (1-y^{(t,i-1)}_{u^{(t,i)}})\mathbf{1}_{u^{(t,i)}}$. \label{step:update-nonmonotone}
	\State Let $B^{(t,i)}\gets B^{(t,i-1)}\cup \{u^{(t,i)}\}$.\;
	\EndFor
	\State $y^{(t+1,0)} = y^{(t,r)}$.
	\EndFor 
	\State \Return $\textsc{Swap-Rounding}(y^{(T,r)},\I)$.%\setminus\N'$.
\end{algorithmic}
\end{algorithm}

The main result of this section is the following.
\begin{theorem}\label{th:main-nonmonotone}
	Let $f:2^{\N}\rightarrow \mathbb{R}_+$, where $|\N|=n$, be a non-monotone, $1$-decomposable, submodular function and $\M=(\N,\I)$ a matroid of rank $r$. Algorithm~\ref{alg:MeasContGreedy} with parameters $\eta$ and $\gamma$ is $(\eps,\delta)$-differentially private and returns a set $S\in\I$ such that, with probability $1-\gamma$, \[\ex{}{f(S)}\geq (1/e-O(\eta))f(\OPT)-O\left(\frac{r}{\eta\eps}\log\frac{nr}{\eta\gamma}\cdot\log\frac{1}{\delta}\right).\] Algorithm~\ref{alg:MeasContGreedy} makes $O\left(\frac{nr^4}{\eta^8}\log\frac{n}{\gamma}\right)$ oracle calls.
\end{theorem}
	
We prove the theorem by combining the utility and privacy guarantees of our algorithm, as stated in Theorem~\ref{thm:utility-nonmonotone} and Theorem~\ref{thm:privacy-nonmonotone}, respectively. We remark that Theorem~\ref{thm:utility-monotone} lower bounds the utility of the fractional solution $F(y^{(T,r)})$. Since $y^{(T,r)}\leq\sum_{t=1}^T\eta \mathbf{1}_{B^{(t,r)}}$, where $B^{(t,r)}\in\I$ for all $t\in[T]$, and the polytope $\mathcal{P}(\M)$ is down-closed, it follows that $y^{(T,r)}\in \mathcal{P}(\M)$ and Lemma~\ref{lem:rounding} can be applied to yield the final guarantees of the integral solution returned by the swap-rounding process. Furthermore, removing any elements of $\N'$ from the solution does not change its value. %\footnote{Note that since $1+\eta\geq T\eta\geq 1$, we might need to scale down $y^{(T,r)}$ by $(1+\eta)$ to ensure that $y^{(T,r)}\in \mathcal{P}(\M)$. This would only change the approximation factor of the utility guarantee by a constant in the term $O(\eta)$.} We also note that the number of oracle calls corresponds to the number of samples $s$ used to estimate the $n$ marginal gains in each of the $Tr$ rounds of the algorithm. The number of oracle calls made by the Swap-Rounding process is $O(\frac{r^2}{\eta})$ and does not increase the overall bound asymptotically.
	
\subsection{Function Estimation}
Let us define $w_D^{(t,i)}(u)=F(y^{(t,i-1)}+\eta(1-y^{(t,i-1}_u)\mathbf{1}_u)-F(y^{(t,i-1)})$ for this section. Next, we prove the sampling error on the new weights $\tw_D^{(t,i)}(u)$ as defined in line~\ref{step:weights-nonmonotone} of Algorithm~\ref{alg:MeasContGreedy}. Note that since the function is non-monotone, the estimated marginal gains have larger variance (they can be possibly negative), so the number of samples needed to ensure that they concentrate around their expected value is larger.
\begin{lemma}\label{lem:gainsamplingerror-nonmonotone}
	With probability at least $1-2\gamma$, for any sequence of points picked by the algorithm $\left\{\{u^{(t,i)}\}_{i=1}^r\right\}_{t=1}^T$ and any $u\in\N$, it holds that
	\[(1-\eta)w_D^{(t,i)}(u)-\frac{\eta f(\OPT)}{rT}\leq \tw_D^{(t,i)}(u)\leq (1+\eta)w_D^{(t,i)}(u)+\frac{\eta f(\OPT)}{rT}.\]
\end{lemma}
\begin{proof}
Fix a point $y^{(t,i-1)}$ and element $u\in\N$. Then $\tw_D^{(t,i)}=G(y^{(t,i-1)}+\eta(1-y^{(t,i-1)}_u)\mathbf{1}_u)-G(y^{(t,i-1)})=\frac{1}{s}\sum_{j=1}^s X_j$ where we define
\[X_j=f(\{v\in\N: r^j_v<y^{(t,i-1)}_v+\eta (1-y^{(t,i-1)}_u) \mathbbm{1}\{v=u\}\})-f(\{v\in\N: r^j_v<y^{(t,i-1)}_v\}).\]
The random variables $X_j$ are independent since the random vectors $r^j$ are also independent. Let us denote the random set based on the vector $r^j$ as $R^j(y^{(t,i-1)})=\{v\in\N: r^j_v<y^{(t,i-1)}_v\}$. Then either $X_j=0$ or $X_j=f(R^j(y^{(t,i-1)})\cup\{u\})-f(R^j(y^{(t,i-1)}))$. We focus on the latter. First, we prove the upper bound:
\begin{align*}
	X_j &\leq f(\{u\})-f(\emptyset) \tag{by submodularity}\\
	& \leq f(\{u\}) \tag{by non-negativity}\\
	& \leq f(\OPT) \tag{since $u\in\I$}
\end{align*}	
We now prove the lower bound:
\begin{align*}
    X_j 
    & = f(R^j(y^{(t,i-1)})\cup\{u\})-f(R^j(y^{(t,i-1)}))\\
    & \geq -f(R^j(y^{(t,i-1)})) \tag{by non-negativity}\\
    & = -f\left(\bigcup_{t=1}^T (B^{(t,r)}\cap R^j(y^{(t,i-1)}))\right) \tag{since $R^j(y^{(t,i-1)})\subseteq \bigcup_{t=1}^T B^{(t,r)}$} \\
    & \geq -\sum_{t=1}^T f(B^{(t,r)}\cap R^j(y^{(t,i-1)})) \tag{by submodularity} \\
    & \geq -T f(\OPT) \tag{since $f(S)\leq f(\OPT)$ for any subset $S$ of a basis}
\end{align*}
Thus $X_j\in[-Tf(\OPT),f(\OPT)]$.
		
Since Lemma~\ref{lem:expectationG} does not assume monotonicity of $f$, it still holds. Therefore, by linearity of expectation, it also holds that \begin{equation}\label{eq:expgain-nonmonotone}
	\ex{}{\tw_D^{(t,i)}(u)}=\frac{1}{s}\sum_{j=1}^s\ex{}{X_j}=F(y^{(t,i-1)}+\eta(1-y^{(t,i-1}_u)\mathbf{1}_u)-F(y^{(t,i-1)}) = w_D^{(t,i)}(u)\end{equation}
		
We recall that Corollary~\ref{cor:concentration} gives us a concentration bound for the mean of $s$ random variables bounded in $[0,R]$. Here, we have that each $X_j\in[r_1,r_2]$, where $r_1 = -Tf(\OPT)$, $r_2 = f(\OPT)$, and $R=r_2-r_1 = (T+1)f(\OPT)$. Applying Corollary~\ref{cor:concentration} to the shifted random variables $X_j'=X_j-r_1\in[0,R]$ and substituting for $s = 48 r^3 T^7 \log \frac{n}{\gamma}> 12 \frac{r^2 T^3 (T+1)}{\eta^2} \log \frac{n^{rT+1}}{\gamma}$, $\alpha = \frac{\eta}{2rT^2}$, and $\beta = \frac{\eta f(\OPT)}{2rT}$, we get:
	\begin{align*}
	   &\pr{}{\tw_D^{(t,i)}(u)<(1-\alpha) w_D^{(t,i)}(u) -\beta + \alpha r_1} \leq \exp(-s\alpha \beta/R)\\
	   \Rightarrow &\pr{}{\tw_D^{(t,i)}(u) < (1-\frac{\eta}{2rT^2})w_D^{(t,i)}(u) - \frac{\eta f(\OPT)}{2rT} - \frac{\eta}{2rT^2} \cdot Tf(\OPT)} \leq \exp\left(\frac{-s \frac{\eta}{2rT^2} \frac{\eta f(\OPT)}{2rT}}{(T+1)f(\OPT)}\right)\\
	    \Rightarrow &\pr{}{\tw_D^{(t,i)}(u)<(1-\frac{\eta}{2rT^2})w_D^{(t,i)}(u)-\frac{\eta f(\OPT)}{rT}} \leq \exp\left(\frac{-s\eta^2}{4 r^2 T^3(T+1)}\right)\\
	    \Rightarrow &\pr{}{\tw_D^{(t,i)}(u)<(1-\frac{\eta}{2rT^2})w_D^{(t,i)}(u)-\frac{\eta f(\OPT)}{rT}} \leq \frac{\gamma}{n^{rT+1}}.
	    \end{align*}
	    Similarly,
	    \begin{align*}
	    &\pr{}{\tw_D^{(t,i)}(u)>(1+\alpha)w_D^{(t,i)}(u) + \beta - \alpha r_1} \leq \exp(-s\alpha\beta/3R)\\
	    \Rightarrow &\pr{}{\tw_D^{(t,i)}(u)>(1+\frac{\eta}{2rT^2})w_D^{(t,i)}(u) + \frac{\eta f(\OPT)}{2rT} + \frac{\eta}{2rT^2} \cdot T f(\OPT)} \leq \exp\left(\frac{-s \frac{\eta}{2rT^2} \frac{\eta f(\OPT)}{2rT}}{3(T+1)f(\OPT)}\right) \\
	    \Rightarrow &\pr{}{\tw_D^{(t,i)}(u)>(1+\frac{\eta}{2rT^2})w_D^{(t,i)}(u)+\frac{\eta f(\OPT)}{rT}} \leq \exp\left(\frac{-s\eta^2}{12 r^2 T^3(T+1)}\right)\\
	    \Rightarrow &\pr{}{\tw_D^{(t,i)}(u)>(1+\frac{\eta}{2rT^2})w_D^{(t,i)}(u)+\frac{\eta f(\OPT)}{rT}} \leq \frac{\gamma}{n^{rT+1}}.
	  \end{align*}
	    
  It follows by union bound that with probability $1-2\gamma/n^{rT+1}$, 
    \[\left(1-\frac{\eta}{2rT^2}\right)w_D^{(t,i)}(u)-\frac{\eta f(\OPT)}{rT}\leq \tw_D^{(t,i)}(u)\leq \left(1+\frac{\eta}{2rT^2}\right)w_D^{(t,i)}(u)+\frac{\eta f(\OPT)}{rT}.\]
    The number of all possible sequences of points $y^{(t,i)}$ is at most the number of possible sequences of picked points $u^{(t,i)}$, which is bounded by $n^{rT}$. By union bound over all sequences of points $y^{(t,i)}$ and elements $u\in\N$, the inequalities hold with probability at least $1-2\gamma$.
\end{proof}
	
\subsection{Utility}
We are now ready to prove the utility guarantee of our algorithm. Note that, by the dummy elements added in the beginning of the algorithm, we can assume that for every round $(t,i)$, $\max_{u\in\N^{(t,i)}} \tw_D^{(t,i)}(u)\geq 0$ and $|B^{(t,r)}|=r$. 

Also, note that Claims~\ref{claim:sensitivity} and~\ref{claim:expmecherror} of Section~\ref{sec:monotone}, which establish the guarantees of the Exponential Mechanism, hold for the new weights $\tw^{(t,i)}(u)$ as their proofs follow exactly the same steps.\footnote{Since adding the dummy elements increases the size of the ground set to $2n$, the additive error of the Exponential Mechanism is larger than the bound of Claim~\ref{claim:expmecherror} by $\frac{2}{\eps_0}\log(2)$. We omit it since it only at most doubles the additive error of the final solution.}

\begin{theorem}\label{thm:utility-nonmonotone}
With probability at least $1-3\gamma$, \[F(y^{(T,0)}) \geq (1/e-O(\eta))f(\OPT)-\frac{8r}{\eta \eps_0}\log\frac{nr}{\eta \gamma}.\]
\end{theorem}
\begin{proof}
	We condition on the event that for all rounds $t\in[T]$, $i\in[r]$, and elements $u\in\N$,
	\[\left(1-\frac{\eta}{2rT^2}\right)w_D^{(t,i)}(u)-\frac{\eta f(\OPT)}{rT}\leq \tw_D^{(t,i)}(u)\leq \left(1+\frac{\eta}{2rT^2}\right)w_D^{(t,i)}(u)+\frac{\eta f(\OPT)}{rT}.\]
	By Lemma~\ref{lem:gainsamplingerror-nonmonotone}, this is true with probability at least $1-2\gamma$. It holds that 
	\begin{align*}
	F(y^{(t+1,0)})-F(y^{(t,0)}) & = F(y^{(t,r)})-F(y^{(t,0)}) \\
	& =\sum_{i=1}^r [F(y^{(t,i)})-F(y^{(t,i-1)})] \\
	& =\sum_{i=1}^r w^{(t,i)}_D(u^{(t,i)}) \tag{since $y^{(t,i)} = y^{(t,i-1)}+\eta(1-y^{(t,i-1)}_{u^{(t,i)}})\mathbf{1}_{u^{(t,i)}}$}\\
	& \geq \sum_{i=1}^r \frac{1}{1+\frac{\eta}{2rT^2}}\tw^{(t,i)}_D(u^{(t,i)}) - r\frac{\eta f(\OPT)}{(1+\frac{\eta}{2rT^2}) r T} \tag{by Lemma~\ref{lem:gainsamplingerror-nonmonotone}}\\
	%& = \frac{1}{1+\frac{\eta}{2rT^2}}\sum_{i=1}^r \tw^{(t,i)}_D(u^{(t,i)}) - \frac{\eta f(\OPT)}{(1+\frac{\eta}{2rT^2}) T}\\
	& \geq \frac{1}{1+\frac{\eta}{2rT^2}}\sum_{i=1}^r \tw^{(t,i)}_D(u^{(t,i)}) - \frac{\eta f(\OPT)}{T}
	\end{align*}
		
	By Lemma~\ref{lem:mappingtoOPT}, there exists a bijection $\phi$ such that $\phi(u^{(t,i)})=o^{(t,i)}$, where $\OPT=\{o^{(t,1)}, \ldots, o^{(t,r)}\}$ is the optimal solution. Now, note that since $o^{(t,i)}$ is a feasible option in the $i$-th round, by the guarantees of the Exponential Mechanism (Claim~\ref{claim:expmecherror}), and by our assumption, with probability $1-\gamma$ we have that
	\begin{equation*}
	\tw_D^{(t,i)}(u^{(t,i)})\geq \max(0,\tw_D^{(t,i)}(o^{(t,i)}))-\frac{2}{\eps_0}\log\frac{nrT}{\gamma},
	\end{equation*}
	for all rounds $i\in[r]$ and $t\in[T]$. We condition on this event for the rest of the proof.
	
	It follows that with probability $1-3\gamma$,
	\begin{align}\label{eq:Ftdiff-nonmonotone}
	F(y^{(t+1,0)})-F(y^{(t,0)})
	& \geq \frac{1}{1+\frac{\eta}{2rT^2}}\sum_{i=1}^r \max(0,\tw_D^{(t,i)}(o^{(t,i)})) - \frac{\eta f(\OPT)}{T} - \frac{2r}{\eps_0}\log\frac{nrT}{\gamma} \nonumber \\
	& \geq \frac{1-\frac{\eta}{2rT^2}}{1+\frac{\eta}{2rT^2}}\sum_{i=1}^r \max(0,w_D^{(t,i)}(o^{(t,i)}))- \frac{2\eta f(\OPT)}{T}-\frac{2r}{\eps_0}\log\frac{nrT}{\gamma}
	\end{align}

To bound the LHS of inequality~\eqref{eq:Ftdiff-nonmonotone}, we prove the following claim.
\begin{claim}\label{claim:sumweightsOPT-nonmonotone}
	For all $t\in[T]$, 
	\[\sum_{i=1}^r \max(0,w_D^{(t,i)}(o^{(t,i)}))\geq \eta [F(y^{(t,r)}\vee\mathbf{1}_{\OPT})-F(y^{(t,r)})].\]
	%Moreover, $F(y^{(t,r)}\vee\mathbf{1}_{\OPT})\geq (e^{-(t+1)}-O(\eta))f(\OPT)$.
	\end{claim}
	\begin{proof}[Proof of Claim~\ref{claim:sumweightsOPT-nonmonotone}]
	By the definition of the weights, it holds that
	\begin{align*}
	\sum_{i=1}^r \max(0,w_D^{(t,i)}(o^{(t,i)})) 
	& = \sum_{i=1}^r \max(0,F(y^{(t,i-1)}+\eta(1-y^{(t,i-1)}_{o^{(t,i)}})\mathbf{1}_{o^{(t,i)}})-F(y^{(t,i-1)}))\\
	& = \sum_{i=1}^r \max(0, \eta(1-y^{(t,i-1)}_{o^{(t,i)}})[F(y^{(t,i-1)}\vee \mathbf{1}_{o^{(t,i)}})-F(y^{(t,i-1)}\wedge \mathbf{1}_{\bar{o}^{(t,i)}})])\\
	& = \sum_{i=1}^r \max(0,\eta[F(y^{(t,i-1)}\vee \mathbf{1}_{o^{(t,i)}})-F(y^{(t,i-1)})])\\
	& \geq \eta \sum_{i=1}^r [F(y^{(t,r)}\vee \mathbf{1}_{o^{(t,i)}})-F(y^{(t,r)})]\\
	& \geq \eta [F(y^{(t,r)}\vee \mathbf{1}_{\OPT})-F(y^{(t,r)})],
	\end{align*}
	where the inequalities follow by submodularity and the fact that $y^{(t,r)}\geq y^{(t,i-1)}$ for all $i\in[r]$. Thus, the claim has been proven.
\end{proof}
		
Next, we invoke the following lemma.
\begin{lemma}[{\citealt[][Lemma~3.2.8]{FeldmanThesis}}]\label{lem:FyorOPT}
	Consider a vector $y\in[0,1]^{\N}$. Assuming $y_u\leq a$ for all $u\in\N$, then for every set $S\subseteq\N$, $F(y\vee\mathbf{1}_S)\geq (1-a)f(S)$. 
\end{lemma}
			
In the next claim, we prove that for all $t\in[T+1]$ and $u\in\N$, $y^{(t,0)}_u$ satisfies the assumption of the lemma for the corresponding parameter $a=1-e^{-\eta t} + O(\eta)$. We prove the Claim~\ref{claim:boundony} following the proof of Lemma 3.2.9 by~\citet{FeldmanThesis}.
	\begin{claim}\label{claim:boundony}
	For every $t\in[T+1]$ and $u\in\N$, it holds that $y^{(t,0)}_u\leq 1-e^{-\eta t} + O(\eta)$.
	\end{claim}
	\begin{proof}
	We will first prove by induction that $y^{(t,0)}_u\leq 1-(1-\eta)^{t}$ for every $u\in\N$. For $t=1$, $y^{(1,0)}_u=0\leq 1-(1-\eta)=\eta$. Let us assume that $y^{(t,0)}_u\leq 1-(1-\eta)^{t}$. We will prove the bound $y^{(t+1,0)}_u\leq 1-(1-\eta)^{t+1}$. Note that in each round, $y^{(t,0)}_u$ will either not increase or, if $u=u^{(t,i)}$ for some $i\in[r]$, it will increase by $\eta(1-y^{(t,i-1)}_u)$. It follows that for any $u\in\N$
	\begin{align*}
	    y^{(t+1,0)}_u &\leq y^{(t,0)}_u+\eta(1-y^{(t,i-1)}_u)\\
	    &\leq y^{(t,0)}_u+\eta(1-y^{(t,0)}_u) \tag{since $y^{(t,0)}_u\leq y^{(t,i-1)}_u$}\\
	    & = (1-\eta)y^{(t,0)}_u+\eta\\
	    & \leq (1-\eta)(1-(1-\eta)^t)+\eta \tag{by the inductive hypothesis}\\
	    & = (1-\eta)-(1-\eta)^{t+1}+\eta\\ 
	    & = 1-(1-\eta)^{t+1}.
	\end{align*}
	
To prove the claim, it remains to prove that for any $t\in[T+1]$, $1-(1-\eta)^{t}\leq 1-e^{\eta t}+O(\eta)$. It holds that
\begin{align*}
    1-(1-\eta)^{t} & \leq 1-[e^{-1}(1-\eta)]^{\eta t}\\
    & = 1-e^{-\eta t}(1-\eta)^{\eta t}\\
    & \leq 1-e^{-\eta t}(1-(T+1)\eta^2)\\
    &\leq 1-e^{-\eta t}+O(\eta),
\end{align*}
by using standard inequalities. This concludes the proof of the claim.
\end{proof}

By Claim~\ref{claim:boundony}, which we just proved, the condition of Lemma~\ref{lem:FyorOPT} holds for $x=y^{(t,r)}=y^{(t+1,0)}$ and $S=\OPT$. Thus, \[F(y^{(t,r)}\vee\mathbf{1}_{\OPT})\geq (e^{-(t+1)}-O(\eta))f(\OPT).\]

Combining this bound with Claim~\ref{claim:sumweightsOPT-nonmonotone}, it follows that 
\[\sum_{i=1}^r \max(0,w_D^{(t,i)}(o^{(t,i)})) \leq (e^{-(t+1)}-O(\eta))f(\OPT).\]

Finally, by the latter and returning to inequality~\eqref{eq:Ftdiff-nonmonotone}, we have showed that
\begin{align*}
	F(y^{(t+1,0)})&-F(y^{(t,0)})\\
	& \geq \frac{1-\frac{\eta}{2rT^2}}{1+\frac{\eta}{2rT^2}}\eta [(e^{-\eta (t+1)}-O(\eta))f(\OPT)-F(y^{(t+1,0)})] - \frac{2\eta f(\OPT)}{T}-\frac{2r}{\eps_0}\log\frac{nrT}{\gamma} \\
	& \geq \eta [(1-2 \cdot \frac{\eta}{2rT^2})(e^{-\eta (t+1)}-O(\eta))f(\OPT)-F(y^{(t+1,0)})] - \frac{2\eta f(\OPT)}{ T}-\frac{2r}{\eps_0}\log\frac{nrT}{\gamma}\\
	%& = \eta [(1-2\eta)e^{-\eta (t+1)}f(\OPT)-F(y^{(t+1,0)})] - O(\eta^2)(1-2\eta)f(\OPT)-\frac{2\eta f(\OPT)}{ T}-\frac{2r}{\eps_0}\log\frac{nrT}{\gamma}\\
	& \geq \eta [(1-\frac{\eta}{rT^2})e^{-\eta (t+1)}f(\OPT)-F(y^{(t+1,0)})] -O(\eta)\frac{ f(\OPT)}{ T}-\frac{2r}{\eps_0}\log\frac{nrT}{\gamma}
\end{align*}
where the second to last inequality holds by non-negativity.
		
To simplify and solve this recursive relation, let us denote $\Omega = (1-\frac{\eta}{rT^2})f(\OPT)$ and $\xi=O(\eta)\frac{ f(\OPT)}{ T}+\frac{2r}{\eps_0}\log\frac{nrT}{\gamma}$. Then, we can rewrite the inequality as follows:
	\begin{align}\label{eq:resursiveF}
	& F(y^{(t+1,0)})-F(y^{(t,0)}) \geq \eta[e^{-\eta(t+1)}\Omega-F(y^{(t+1,0)})]-\xi \nonumber \\
	\Leftrightarrow &F(y^{(t+1,0)}) \geq \frac{F(y^{(t,0)})}{1+\eta} +\frac{\eta e^{-\eta(t+1)}}{1+\eta}\Omega-\frac{\xi}{1+\eta}
	\end{align}
	We will prove by induction that
	\begin{equation}\label{eq:inductivehypothesisF}
	F(y^{(t,0)}) \geq \eta(t-1)e^{-\eta t} \Omega -\eta^2(t-1)\Omega - (t-1)\xi.
	\end{equation}
	
	For $t=1$, $F(y^{(1,0)})=f(\emptyset)\geq 0$, which satisfies the inequality. Let us assume that inequality~\eqref{eq:inductivehypothesisF} holds. Then
	\begin{align*}
	 F(y^{(t+1,0)}) &\geq \frac{F(y^{(t,0)})}{1+\eta} +\frac{\eta e^{-\eta(t+1)}}{1+\eta}\Omega-\frac{\xi}{1+\eta} \tag{by~\eqref{eq:resursiveF}}\\
	 & \geq \frac{\eta(t-1)e^{-\eta t}}{1+\eta} \Omega -\frac{\eta^2(t-1)}{1+\eta}\Omega - \frac{(t-1)\xi}{1+\eta}+\frac{\eta e^{-\eta(t+1)}}{1+\eta}\Omega-\frac{\xi}{1+\eta} \tag{by~\eqref{eq:inductivehypothesisF}}\\
	 & = \eta e^{-\eta(t+1)}\left(\frac{e^\eta (t-1)}{1+\eta}+\frac{\eta e^{\eta(t+1)}}{1+\eta}+\frac{1}{1+\eta}\right)\Omega-\frac{\eta^2 t}{1+\eta}\Omega-\frac{\xi t}{1+\eta}\\
	 & \geq \eta e^{-\eta(t+1)}\left(\frac{e^\eta (t-1)}{1+\eta}+\frac{\eta e^{\eta(t+1)}}{1+\eta}+\frac{1}{1+\eta}\right)\Omega -\eta^2t\Omega-\xi t\\
	 & \geq \eta e^{-\eta(t+1)}\left((t-1)+\frac{\eta}{1+\eta}+\frac{1}{1+\eta}\right)\Omega -\eta^2t\Omega-\xi t \tag{since $e^\eta \geq 1+\eta$}\\
	 & \geq \eta e^{-\eta(t+1)}t\Omega -\eta^2t\Omega-\xi t.
	\end{align*}
Thus, inequality~\eqref{eq:inductivehypothesisF} holds for all $t\in[T+1]$. For $t=T+1$, substituting $\Omega$ and $\xi$, we get
	\begin{align*}
	F(y^{(T+1,0)}) & \geq \eta T e^{-\eta (T+1)} \Omega -\eta^2T\Omega - T\xi\\
	& \geq e^{-2\eta-1} \Omega -2\eta\Omega - T\xi \tag{since $1+\eta\geq \eta T\geq 1$} \\
	& \geq e^{-2\eta-1}(1-\frac{\eta}{rT^2})f(\OPT)-2\eta(1- \frac{\eta}{rT^2})f(\OPT)-O(\eta)f(\OPT)-\frac{2rT}{\eps_0}\log\frac{nrT}{\gamma}\\
	& \geq \frac{(1-2\eta)}{e}(1-\frac{\eta}{rT^2})f(\OPT)-O(\eta)f(\OPT)-\frac{2rT}{\eps_0}\log\frac{nrT}{\gamma} \tag{since $e^{-x}\geq 1-x$}\\
	&\geq (1/e-O(\eta))f(\OPT)-\frac{2rT}{\eps_0}\log\frac{nrT}{\gamma}\\
	& \geq (1/e-O(\eta))f(\OPT)-\frac{8r}{\eta \eps_0}\log\frac{nr}{\eta \gamma}. \tag{since $T=\lceil \frac{1}{\eta}\rceil\leq \frac{2}{\eta}$}
	\end{align*}
This concludes the proof of the lemma.
\end{proof}

\subsection{Privacy Analysis}
The privacy analysis of Algorithm~\ref{alg:MeasContGreedy} follows the same structure as the one of the previous section, but the fact that $f$ is non-monotone requires a few new key observations. 

Intuitively, in the monotone case, we managed to bound the ratio of the privacy loss by a function of a sum of expected marginal gains and rely on the concentration of Claim~\ref{claim:privacyconc} to prove that this sum is bounded by the sum of realized marginal gains, which, in turn, was at most $1$ since the function $f_I$ is monotone and $1$-decomposable. 

In the non-monotone case, the ratio of the privacy loss can again be bounded by a function of a sum of expected \emph{absolute} marginal gains.\footnote{Note that this is also the case for monotone functions, but the absolute marginal gain is the same as the marginal gain, as it is always positive.} The next lemma and its corollary allow us to bound the sum of realized \emph{absolute} marginal gains, which we will use to invoke Claim~\ref{claim:privacyconc} again, to establish the privacy guarantees of Theorem~\ref{thm:privacy-nonmonotone} at the end of this section. Bounding this sum is not as trivial as in the monotone case: the ``movement'' of a non-monotone function could be unbounded, even though the function has a bounded range, so we have to leverage the fact that $f_I$ is submodular. The main idea is to use submodularity to bound the total ``increase'' of the function, which, since the function has a bounded range, leads to a bound of the total ``decrease'' of the function as well.

\begin{lemma}\label{lem:nonMonIneq} Let $f_I : 2^{[n]} \to [0,1]$ be a submodular function. Then \[\sum_{i=1}^{n} |f_I(\{1,\dots, i\}) - f_I(\{1 , \dots , i-1\})| \leq 2-f_I(\emptyset).\]
Moreover, since the order of the elements of $[n]$ is arbitrary, by the triangle inequality, this implies that for any sequence of non-decreasing sets $\emptyset = T_0 \subseteq T_1 \subseteq \dots \subseteq T_r \subseteq [n]$, \[\sum_{i=1}^{r} |f_I(T_{i}) - f_I(T_{i-1})| \leq 2-f_I(\emptyset).\]
\end{lemma}
\begin{proof}
Let $S_i=\{1, \dots, i\}$. Suppose $T=\{i_t : t = 1, \dots, k\}$, for some $k\in[n]$, is the set of indices for which $f_I(S_{i_t}) - f_I(S_{i_t-1}) \geq 0$. Then, by submodularity, for $t= 1, \dots , k$ we have that
\[ f_I(S_{i_t}) - f_I(S_{i_t-1}) \leq f_I(T\cap S_{i_t})-f_I(T\cap S_{i_t-1}) = f_I(i_1 ,\dots , i_t) - f_I(i_1 ,\dots , i_{t-1}). \] 
Summing over the range of $t\in[k]$, it follows that
\begin{align}
	&\sum_{t=1}^k f_I(S_{i_t}) - f_I(S_{i_t - 1})  \leq f_I(i_1, \dots i_k) - f_I(\emptyset) \leq 1 - f_I(\emptyset)\nonumber \\
	\Rightarrow &\sum_{t=1}^k |f_I(S_{i_t}) - f_I(S_{i_t - 1}) | \leq 1 - f_I(\emptyset) \label{eq:nmeq1}
\end{align}
		
Similarly, we let $j_1 ,\dots j_\ell$ be the indices for which $f_I(S_{j_t}) - f_I(S_{j_t-1}) < 0$. Then
\begin{align}
	&\sum_{t=1}^k f_I(S_{i_t}) - f_I(S_{i_t - 1}) + \sum_{t=1}^\ell f_I(S_{j_t}) - f_I(S_{j_t - 1}) = \sum_{i=1}^{n} f_I(S_{i}) - f_I(S_{i-1}) \nonumber \\
	\Rightarrow &\sum_{t=1}^\ell f_I(S_{j_t}) - f_I(S_{j_t - 1}) = f_I([n]) -  f_I(\emptyset) -\sum_{t=1}^k f_I(S_{i_t}) - f_I(S_{i_t - 1}) \nonumber \\
	\Rightarrow &\sum_{t=1}^\ell f_I(S_{j_t}) - f_I(S_{j_t - 1})\geq -1 \tag{by~\eqref{eq:nmeq1}}\nonumber \\
	\Rightarrow &\sum_{t=1}^\ell | f_I(S_{j_t}) - f_I(S_{j_t - 1}) |\leq 1 \label{eq:nmeq2}
\end{align}
		
Adding inequalities~\eqref{eq:nmeq1} and~\eqref{eq:nmeq2}, we get the result.
\end{proof}
	
\begin{corollary}\label{cor:estSumBound}
Consider any sequence of elements $v^{(t,i)}$ and solutions $y^{(t,i)}$ so that $y^{(t,i)}=y^{(t,i-1)}+\eta(1-y^{(t,i-1)})\mathbf{1}_{v^{(t,i)}}$. Then \[\sum_{t=1}^T \sum_{i=1}^r \left|G_I(y^{(t,i-1)} + \eta(1-y^{(t,i-1)})\mathbf{1}_{v^{(t,i)}}) - G_I(y^{(t,i-1)})\right| \leq 2-f_I(\emptyset).\]
\end{corollary}
\begin{proof}
\begin{align*}
	&\sum_{t=1}^T \sum_{i=1}^r \left|G_I(y^{(t,i-1)} + \eta(1-y^{(t,i-1)})\mathbf{1}_{v^{(t,i)}}) - G_I(y^{(t,i-1)})\right| \\
	&= \sum_{t=1}^T \sum_{i=1}^r \left| \sum_{j=1}^s \frac{f_I(\{ u: r_u^j< y^{(t,i)} \}) - f_I(\{ u: r_u^j< y^{(t,i-1)} \})}{s} \right|  \tag{by the definition of $y^{(t,i)}$}\\
	%&\leq \sum_{t=1}^T \sum_{i=1}^r  \sum_{j=1}^s \frac{|f_I(\{ u: r_u^j< y^{(t,i)} \}) - f_I(\{ u: r_u^j< y^{(t,i-1)} \})|}{s} \\
	&\leq  \frac{1}{s} \sum_{j=1}^s \sum_{t=1}^T \sum_{i=1}^r  |f_I(\{ u: r_u^j< y^{(t,i)} \}) - f_I(\{ u: r_u^j< y^{(t,i-1)} \})| \tag{by the triangle inequality}
\end{align*}
We observe that the sequence of sets $\{ u: r_u^j< y^{(t,i)} \}$ that occur as the arguments of $f_I$ is a non-decreasing sequence of sets, and hence Lemma~\ref{lem:nonMonIneq} applies to give us
\begin{align*}
	\sum_{t=1}^T \sum_{i=1}^r \left|G_I(y^{(t,i-1)} + \eta(1-y^{(t,i-1)})\mathbf{1}_{u^{(t,i)}}) - G_I(y^{(t,i-1)})\right|
	\leq \frac{1}{s}\sum_{j=1}^s (2 -f_I(\emptyset))
	=2-f_I(\emptyset)
\end{align*}
This concludes the proof of the corollary.
\end{proof}

We are new ready to prove the next theorem which establishes the privacy guarantees of our algorithm.
\begin{theorem}\label{thm:privacy-nonmonotone}
Algorithm~\ref{alg:MeasContGreedy} is $((14+4\log \frac{1}{\delta})\eps_0, \delta)$-differentially private.
\end{theorem}
\begin{proof}
Let $A$ and $B$ be two sets of agents such that $A \triangle B = \{I\}$. Suppose that instead of the output set, we reveal the sequence in which we pick the elements of our algorithm and let this sequence be denoted as $U = (u^{(1,1)}, u^{(1,2)}, \dots, u^{(T,r)})$. Note that the sequence $U$ might include some dummy elements. 
We are then interested in bounding the ratio of the probabilities that the output sequence be $U$ under input $A$ and $B$. 
By the post-processing property (Lemma~\ref{lem:post-processing}), this suffices to achieve the same privacy parameters over the output of the algorithm, which is the set \textsc{Swap-Rounding}$(y^{(T,r)}, \I)$ with any dummy elements removed.
	
We recall that the scores of the Exponential Mechanism, as defined in line~\ref{step:weights-nonmonotone} of Algorithm~\ref{alg:MeasContGreedy}, are $\tw_D^{(t,i)}(u) = G(y^{(t,i-1)} + \eta(1-y^{(t,i-1)})\mathbf{1}_{u})-G(y^{(t,i-1)})$. For ease of notation, we drop the irrelevant parameters of our algorithm and denote it by $M$. By the chain rule of probability,
\begin{align}\label{eq:privacyratio-nm}
	\frac{\pr{}{\textsc{M}(A) = U}}{\pr{}{\textsc{M}(B) = U}}
	&= \prod_{t=1}^T \prod_{i=1}^r \frac{\exp(\frac{\eps_0}{2} \tw_A^{(t,i)}(u^{(t,i)}))/\sum_{u\in\N^{(t,i)}} \exp(\frac{\eps_0}{2} \tw_A^{(t,i)}(u))}{\exp(\frac{\eps_0}{2} \tw_B^{(t,i)}(u^{(t,i)}))/\sum_{u\in\N^{(t,i)}} \exp(\frac{\eps_0}{2} \tw_B^{(t,i)}(u))} \nonumber\\
	&=\left( \prod_{t=1}^T\prod_{i=1}^r \frac{\exp(\frac{\eps_0}{2} \tw_A^{(t,i)}(u^{(t,i)}))}{\exp(\frac{\eps_0}{2} \tw_B^{(t,i)}(u^{(t,i)}))}\right) \left( \prod_{t=1}^T\prod_{i=1}^r \frac{\sum_{u\in\N^{(t,i)}} \exp(\frac{\eps_0}{2} \tw_B^{(t,i)}(u))}{\sum_{u\in\N^{(t,i)}} \exp(\frac{\eps_0}{2} \tw_A^{(t,i)}(u))}\right)
\end{align}
 
We bound these two factors separately. We assume without loss of generality that $A \setminus B = \{I\}$. We upper and lower bound the first factor as follows.
\begin{align*}
	\prod_{t=1}^T\prod_{i=1}^r \frac{\exp(\frac{\eps_0}{2} \tw_A^{(t,i)}(u^{(t,i)}))}{\exp(\frac{\eps_0}{2} \tw_B^{(t,i)}(u^{(t,i)}))}
	%&= \prod_{t=1}^T\prod_{i=1}^r \frac{\exp(\frac{\eps_0}{2} (\tw_I^{(t,i)}(u^{(t,i)})+\tw_B^{(t,i)}(u^{(t,i)})))}{\exp(\frac{\eps_0}{2} \tw_B^{(t,i)}(u^{(t,i)}))} \\
	&=\prod_{t=1}^T\prod_{i=1}^r \exp(\frac{\eps_0}{2} \tw_{I}^{(t,i)}(u^{(t,i)}) )\\
	&=\exp(\sum_{t=1}^T\sum_{i=1}^r \frac{\eps_0}{2} \tw_{I}^{(t,i)}(u^{(t,i)}) ) \\
	&= \exp( \frac{\eps_0}{2} \sum_{t=1}^T\sum_{i=1}^r  G_I(y^{(t,i-1)} + \eta (1-y^{(t,i-1)})\mathbf{1}_{u^{(t,i)}})-G_I(y^{(t,i-1)}) ) \\
	&= \exp(\frac{\eps_0}{2} \left(G_I (y^{(T,r)}) - G_I(y^{(1,0)})\right) ).
\end{align*}
Since $f_I$ has range in $[0,1]$, it follows that 
\begin{equation}\label{eq:privfactor1}
\exp(-\eps_0/2) \leq \prod_{t=1}^T\prod_{i=1}^r \frac{\exp(\frac{\eps_0}{2} \tw_A^{(t,i)}(u^{(t,i)}))}{\exp(\frac{\eps_0}{2} \tw_B^{(t,i)}(u^{(t,i)}))} \leq \exp(\eps_0/2).
\end{equation}

Next, we upper and lower bound the reciprocal of the second factor as follows.
\begin{align}
	&\prod_{t=1}^T \prod_{i=1}^r \frac{\sum_{u \in \N^{(t,i)}} \exp(\frac{\eps_0}{2} \tw_A^{(t,i)}(u))}{\sum_{u \in \N^{(t,i)}} \exp(\frac{\eps_0}{2} \tw_B^{(t,i)} (u))} \nonumber \\
	&= \prod_{t=1}^T \prod_{i=1}^r \frac{\sum_{u \in \N^{(t,i)}}  \exp(\frac{\eps_0}{2} \tw_I^{(t,i)}(u))\exp(\frac{\eps_0}{2} \tw_B^{(t,i)}(u))}{\sum_{u \in \N^{(t,i)}} \exp(\frac{\eps_0}{2} \tw_B^{(t,i)}(u))} \nonumber \\
	&= \prod_{t=1}^T \prod_{i=1}^r  \ex{u\gets P^{(t,i)}}{\exp(\frac{\eps_0}{2} \tw_I^{(t,i)}(u) )} \nonumber \\
	&= \prod_{t=1}^T \prod_{i=1}^r  \ex{u\gets P^{(t,i)}}{\exp(\frac{\eps_0}{2} (G_I(y^{(t,i-1)} + \eta (1-y^{(t,i-1)})\mathbf{1}_{u})-G_I(y^{(t,i-1)})) )} \label{eq:simplecalculation}\\
	&\leq \prod_{t=1}^T \prod_{i=1}^r  \ex{u\gets P^{(t,i)}}{\exp(\frac{\eps_0}{2} \left|G_I(y^{(t,i-1)} + \eta (1-y^{(t,i-1)})\mathbf{1}_{u})-G_I(y^{(t,i-1)})\right| ) } \nonumber \\
	&\leq \exp((e^{\eps_0/2}-1)\sum_{t=1}^T \sum_{i=1}^r \ex{u\gets P^{(t,i)}}{\left|G_I(y^{(t,i-1)} + \eta(1-y^{(t,i-1)})\mathbf{1}_{u}) - G_I(y^{(t,i-1)})\right|} ),\label{eq:expofsum}
	\end{align}
where the last inequality follows by subsequently applying $e^x \leq 1 + \frac{e^{\eps_0/2}-1}{\eps_0/2}\cdot x ~ \forall x\in[0,\frac{\eps_0}{2}]$, and $1 + t \leq e^t ~ \forall t$.

By using Claim~\ref{claim:privacyconc}, we will upper bound the sum of expected absolute marginal gains \[\sum_{(t,i)=(1,1)}^{(T,r)}\ex{u\gets P^{(t,i)}}{\left|G_I(y^{(t,i-1)} + \eta(1-y^{(t,i-1)})\mathbf{1}_{u}) - G_I(y^{(t,i-1)})\right|}.\]

Again, consider a $Tr$-round probabilistic process. In each round $(t,i)$, an adversary chooses a distribution $P^{(t,i)}$ and then a sample $u^{(t,i)}$ is drawn from that distribution. Let \[Z_{(t,i)}=2-f_I(\emptyset)-\sum_{(\tau,j)=(1,1)}^{(t,i-1)} \left|G_I(y^{(\tau,j)}) - G_I(y^{(\tau,j-1)})\right|.\] 

By Corollary~\ref{cor:estSumBound}, it holds that $Z_{(t,i)}\geq 0$. One can think of the quantity $Z_{(t,i)}$ as tracking the remaining progress towards the total sum of the absolute realized values of the marginal gains, namely, \[\sum_{(t,i)=(1,1)}^{(T,r)} \left|G_I(y^{(t,i-1)} + \eta(1-y^{(t,i-1)})\mathbf{1}_{u^{(t,i)}}) - G_I(y^{(t,i-1)})\right|.\]
We also define the random variable \[R_{(t,i)}(u)=\frac{\left|G_I(y^{(t,i-1)} +\eta(1-y^{(t,i-1)})\mathbf{1}_{u})-G_I(y^{(t,i-1)})\right|}{Z_{(t,i)}}.\] The distribution of $R_{(t,i)}(u)$, denoted by $D^{(t,i)}$, is directly determined by $P^{(t,i)}$. Also note that, since $Z_{(t,i)}\geq0$, $R_{(t,i)}(u)\geq0$. By Corollary~\ref{cor:estSumBound},
\[\left|G_I(y^{(t,i-1)} +\eta(1-y^{(t,i-1)})\mathbf{1}_{u})-G_I(y^{(t,i-1)})\right| + \sum_{(\tau,j)=(1,1)}^{(t,i-1)} \left|G_I(y^{(\tau,j)}) - G_I(y^{(\tau,j-1)})\right| \leq 2-f_I(\emptyset).\]
Therefore, $R_{(t,i)}(u)\in[0,1]$.
After the element of the round $u^{(t,i)}$ has been chosen, we update $Z_{(t,i)}=Z_{(t,i-1)}-R_{(t,i-1)}(u^{(t,i)})Z_{(t,i-1)}$. Finally, let $Y_{(\tau,j)}=\sum_{(t,i)\geq(\tau,j)}\ex{u\gets P^{(t,i)}}{R_{(t,i)}(u)}Z_{(t,i)}$. By these definitions, the sum we want to bound is exactly $Y_{(1,1)}$.

By Claim~\ref{claim:privacyconc}, the sum $Y_{(1,1)}$ as defined through this random process satisfies
\begin{align}
	&\pr{}{Y_{(1,1)}\geq qZ_{(1,1)}}\leq \exp(3-q) \nonumber \\
	\Leftrightarrow &\pr{}{Y_{(1,1)}\geq q(2-f_I(\emptyset))}\leq \exp(3-q) \nonumber \\
	\Rightarrow &\pr{}{Y_{(1,1)}\geq 2q}\leq \exp(3-q) \nonumber \\
	\Rightarrow &\pr{}{\sum_{t=1}^T \sum_{i=1}^r \ex{u}{\left|G_I (y^{(t,i-1)} + \eta (1-y^{(t,i-1)})\mathbf{1}_{u}) - G_I(y^{(t,i-1)})\right|} \geq 6+2\log \frac{1}{\delta}} \leq \delta \label{eq:boundsumabs}
\end{align}

Therefore, with probability $1-\delta$, by inequalities~\eqref{eq:expofsum} and~\eqref{eq:boundsumabs},
\begin{equation}\label{eq:privfactor2i}
\prod_{t=1}^T \prod_{i=1}^r \frac{\sum_{u \in \N^{(t,i)}} \exp(\frac{\eps_0}{2} \tw_A^{(t,i)}(u))}{\sum_{u \in \N^{(t,i)}} \exp(\frac{\eps_0}{2} \tw_B^{(t,i)} (u))} \leq \exp\left((e^{\eps_0/2}-1) (6+2\log \frac{1}{\delta}) \right).
\end{equation}

To complete the proof, it remains to lower bound the reciprocal of the second factor of equation~\eqref{eq:privacyratio-nm}. By the same calculations that led to equation~\eqref{eq:simplecalculation} again we have
\begin{align*}
&\prod_{t=1}^T \prod_{i=1}^r \frac{\sum_{u \in \N^{(t,i)}} \exp(\eps_0 \tw_A^{(t,i)}(u))}{\sum_{u \in \N^{(t,i)}} \exp(\eps_0 \tw_B^{(t,i)} (u))} \\
&\geq \prod_{t=1}^T \prod_{i=1}^r  \ex{u\gets P^{(t,i)}}{\exp(- \frac{\eps_0}{2} \left|G_I(y^{(t,i-1)} + \eta (1-y^{(t,i-1)})\mathbf{1}_{u})-G_I(y^{(t,i-1)})\right|)}\\
&\geq \prod_{t=1}^T \prod_{i=1}^r  \ex{u\gets P^{(t,i)}}{1- \frac{\eps_0}{2} \left|G_I(y^{(t,i-1)} + \eta (1-y^{(t,i-1)})\mathbf{1}_{u})-G_I(y^{(t,i-1)})\right|} \tag{$e^x \geq 1+x ~ \forall x$}\\
&\geq \exp(-\eps_0 \sum_{t=1}^T \sum_{i=1}^r \ex{u\gets P^{(t,i)}}{\left|G_I(y^{(t,i-1)} + \eta (1-y^{(t,i-1)})\mathbf{1}_{u})-G_I(y^{(t,i-1)})\right|}). \tag{$1 - t \geq e^{-2t} ~ \forall t\in[0,1/2]$}
\end{align*}
%tighter bound: $\log\left(\frac{1}{1-\eps_0/2}\right)$ in place of $\eps_0$.
By inequality~\eqref{eq:boundsumabs}, we have that with probability $1-\delta$, 
\begin{equation}\label{eq:privfactor2ii}
\prod_{t=1}^T \prod_{i=1}^r \frac{\sum_{u \in \N^{(t,i)}} \exp(\eps_0 \tw_A^{(t,i)}(u))}{\sum_{u \in \N^{(t,i)}} \exp(\eps_0 \tw_B^{(t,i)} (u))} \geq \exp(-\eps_0(6+2\log \frac{1}{\delta})).
\end{equation}

Combining inequalities~\eqref{eq:privfactor1},~\eqref{eq:privfactor2i}, and~\eqref{eq:privfactor2ii}, we conclude that for any sets such that $A\setminus B = \{I\}$, with probability $1-\delta$, the probability ratio is bounded in the range:
\begin{align*}
   \exp\left(-(e^{\eps_0/2}-1)(7+\log\frac{1}{\delta})\right) \leq \frac{\pr{}{\textsc{M}(A) = U}}{\pr{}{\textsc{M}(B) = U}} \leq \exp\left(\eps_0 \left(7+2\log\frac{1}{\delta}\right)\right).
\end{align*}
%tighter upper bound $\log\left(\frac{1}{1-\eps_0/2}\right)$ in place of $\eps_0$.
Thus, for any two neighboring sets such that $A\sim B$, Algorithm~\ref{alg:MeasContGreedy} is $((14+4\log(1/\delta))\eps_0, \delta)$-differentially private.
%tighter $\left(2\log\left(\frac{1}{1-\eps_0/2}\right)\left(7+2\log\frac{1}{\delta}\right), \delta\right)$-differentially private.
\end{proof}

%% file: experiments.tex
\section{Experiments}\label{sec:experiments}
In this section we describe two experiments evaluating the performance of the Private Continuous Greedy (PCG) algorithm of Section~\ref{sec:monotone}. We replicate the Uber location selection experiment in~\citep{MitrovicBKK17}, comparing the PCG algorithm and its rank invariant noise addition with the composition-law based differentially private greedy (DPG) algorithm, introduced in that paper. We also study a hard instance of a partition matroid constraint where PCG significantly outperforms the discrete DPG of~\citep{MitrovicBKK17} in utility, even when the latter uses the rank invariant privacy parameter.

Following ~\cite{MitrovicBKK17}, we consider the problem of picking $r$ public waiting spots for Uber cabs that are close to potential pick-up requests in Manhattan. This is done by using a dataset of Uber pick-ups in April 2014 \citep{Uber}, in which each record contains the latitude and longitude of a pick-up. The goal is to choose a good set of waiting locations which satisfy the given cardinality or matroid constraints, while satisfying differential privacy with respect to the pick-ups, each of which is assumed to represent only one individual.

We first describe the metric used to define the utility of a set of locations. We recall that the normalised $\ell_1$ (or Manhattan) distance between a waiting location $l=(l_x,l_y)$ and a pick-up $p=(p_x,p_y)$ is given by the expression
\begin{equation*}\label{eq:manhattan}
	M(l,p) = \frac{|l_x - p_x| - |l_y - p_y|}{C},
\end{equation*}
where $C$ is a minimal normalisation factor such that $M(l,p) \in [0,1]$ for our choice of locations and pick-ups. We proceed to define the utility of a non-empty set of locations $S$ evaluated on a dataset of pick-ups $D$ as
\begin{equation}\label{eq:defobjective}
	f_D (S) = \sum_{p\in D} \left(1 - \min_{l \in S} M(l,p)\right) = |D| - \sum_{p\in D} \min_{l \in S} M(l,p).
\end{equation}
Setting $f_D(\emptyset) = 0$, it can be checked that the function $f_D$ for any non-empty dataset $D$ is a positive monotone decomposable submodular function.

\subsection{Cardinality constraint}
Our first experiment is for the case of the decomposable monotone submodular function defined in~\eqref{eq:defobjective} and an $r$-cardinality constraint. We compare our PCG algorithm (Algorithm~\ref{alg:ContGreedy}) with the general monotone submodular maximization algorithm DPG of~\cite{MitrovicBKK17}. Their theoretical utility guarantees are $(1-1/e-\eta)f(\OPT)-\frac{r\log n}{\eta\eps}$ and $(1-1/e)f(\OPT)-\frac{r^{3/2}\log n}{\eps}$, respectively.\footnote{As~\citet{MitrovicBKK17} state, their algorithm with the privacy parameter calculated using basic DP composition often performs better than the one that uses advanced. We always check which of the two initializations of DPG performs better and use this for our comparison.} These bounds imply that settings with low rank are more favourable for DPG (for which $\eps_0$ is set according to the basic DP composition as $\approx\eps/r$) and settings with high rank are more favourable for PCG (where $\eps_0\approx\eps$ as in line~\ref{step:eps0} of Algorithm~\ref{alg:ContGreedy}). We want to investigate the performance of PCG in the high rank regime.

We follow~\citet{MitrovicBKK17} and use a grid over downtown Manhattan as potential waiting locations. We note that this cardinality problem is easier than general cardinality constrained submodular maximization. Because of its structure and the density of the points, the problem reduces to an easy instance of a Geometric Maximum Coverage problem, which admits a PTAS~\citep{LiWZZ15}. As is, a large randomly selected set performs well and gets close to the maximum utility. Indeed, in practice we found that a randomly chosen set performs about as well as the other algorithms beyond rank $r=10$ in a direct re-implementation of the experiment. 
To make the instance harder for random selection while virtually the same in difficulty for most algorithms, we add a large number of copies of the northern corner of the grid to the set of choices.

Concretely, to conduct this harder variant of the location selection experiment we choose the set of possible locations as a $5\times 4$ grid of locations in the fixed box, and 80 copies of its northernmost vertex. For each execution, we choose $m=100$ pickups uniformly at random from our dataset and evaluate the empirical utilities of DPG and our PCG. In PCG, we set $\eta = 0.2$ and use $s=1000$ samples to calculate the marginal gains with respect to the function $G$ (i.e., the proxy of $F$). We also measure the performance of the non-private greedy which has optimal utility as a yardstick, and that of a randomly chosen basis set that serves as a trivial private baseline. 
We set $\eps=0.1$ and $\delta = 1/m^{1.5}$ where $m = |D| = 100$. 
For these choices of privacy parameters, the privacy parameter used in the differentially private choices of increment is about $\eps_0=0.01006$. 

In Figure~\ref{fig:cardinality}, we see that for this experiment the PCG algorithm starts to outperform the DPG algorithm around rank $r=10$, but that again both private algorithms become equivalent to picking a uniformly random set around rank $r=20$. It is around $r=10$ that our setting for $\eps_0$ starts to be larger than the rank-sensitive privacy parameter $\eps/r$ used in each round of the DPG algorithm, which justifies this trend. 
For the intermediate rank domain $r \in [10,20]$, both algorithms outperform the utility offered by a random basis set, but fall short of the non-private greedy which has optimal utility, as expected. In this regime, PCG confers an advantage to DPG.

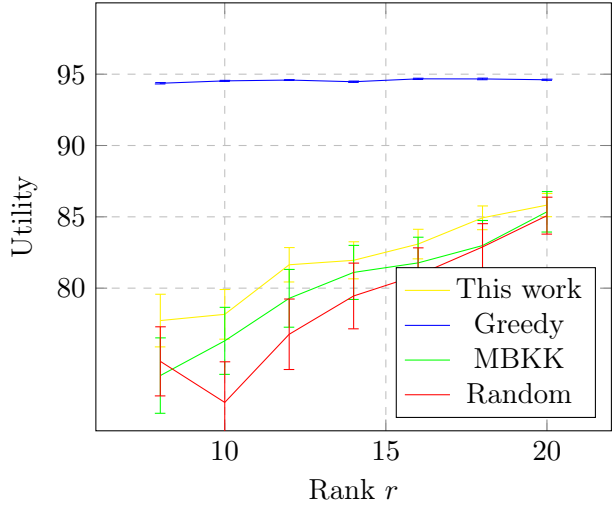
\begin{figure}
\centering
	\begin{tikzpicture}
	\begin{axis}[
	title={Utility versus rank, $\eps = 0.1$},
	xlabel={Rank $r$},
	ylabel={Utility},
	xmin=6, xmax=22,
	ymin=70, ymax=100,
	xtick={10,15,20,25,30,35},
	ytick={80,85,90,95},
	legend pos=south east,
	ymajorgrids=true,
	xmajorgrids=true,
	grid style=dashed,
	]
	
	\addplot[yellow] plot[ 
	error bars/.cd,
	y dir = both, y explicit
	]
	coordinates {
		(8,77.7277) +- (1.0092599999999998,1.8450822669104203)
		(10,78.1646) +- (0.952436,1.741199268738574)
		(12,81.6389) +- (0.6593340000000001,1.205363802559415)
		(14,81.9523) +- (0.710465,1.298839122486289)
		(16,83.0898) +- (0.566725,1.0360603290676418)
		(18,84.9309) +- (0.455372,0.8324899451553931)
		(20,85.8223) +- (0.444712,0.8130018281535649)
	};
	\addplot[blue] plot[ error bars/.cd,
	y dir = both, y explicit
	]
	coordinates {
		(8,94.363) +- (0.026331700000000003,0.048138391224862895)
		(10,94.5305) +- (0.019764200000000003,0.03613199268738575)
		(12,94.5883) +- (0.0146575,0.026796160877513714)
		(14,94.4704) +- (0.0253876,0.046412431444241314)
		(16,94.6733) +- (0.0233854,0.04275210237659964)
		(18,94.6669) +- (0.029973899999999998,0.05479689213893967)
		(20,94.6087) +- (0.022534699999999998,0.04119689213893967)
	};
	\addplot[green] plot[ 
	error bars/.cd,
	y dir = both, y explicit
	]
	coordinates {
		(8,73.871) +- (1.44617,2.64382084095064)
		(10,76.3059) +- (1.28505,2.3492687385740405)
		(12,79.2873) +- (1.10785,2.0253199268738573)
		(14,81.1035) +- (1.03626,1.8944424131627058)
		(16,81.7698) +- (0.9849600000000001,1.800658135283364)
		(18,82.9854) +- (0.962036,1.758749542961609)
		(20,85.3519) +- (0.775712,1.4181206581352834)
	};
	\addplot[red] plot[ 
	error bars/.cd,
	y dir = both, y explicit
	]
	coordinates {
		(8,74.8688) +- (1.32487,2.4220658135283366)
		(10,71.9892) +- (1.55801,2.8482815356489946)
		(12,76.7668) +- (1.35131,2.470402193784278)
		(14,79.4524) +- (1.25996,2.303400365630713)
		(16,80.9212) +- (1.04142,1.9038756855575867)
		(18,82.8864) +- (0.8938729999999999,1.6341371115173675)
		(20,85.0823) +- (0.708121,1.2945539305301645)
	};
	\legend{This work, Greedy, MBKK, Random}
	\end{axis}
	\end{tikzpicture}
\caption{Empirical performance of non-private greedy, PCG, DPG, and random selection under the cardinality constraint.}\label{fig:cardinality}
\end{figure}

\subsection{Partition matroid constraint}
As noted in~\citep{MitrovicBKK17}, in the decomposable case with matroid constraints, DPG combined with the privacy analysis of~\citep{GuptaLMRT10} gives the optimal additive error (see Table~\ref{table}).
In this section, we demonstrate that, even in this case, the $\frac{1}{2}$ multiplicative approximation factor in the DPG guarantee is not a pessimistic upper bound but in fact a tight one. To see why this is, we look at the trace of the non-private discrete greedy in a simple $3$-element partition matroid.

Let $S=\{A,B,C\}$ be the ground set, and $\{\{A\},\{B,C\}\}$ be the partition. We let our matroid structure be a simple rank-$1$ constraint on each partition, i.e. sets in the matroid can have at most one element from each partition. For monotone increasing submodular functions it follows that the choice is essentially between $\{A,B\}$ and $\{A,C\}$. If the submodular function is such that $f(\{B\}) > f(\{A\}), f(\{C\})$ but $f(\{A,B\}) < f(\{A,C\})$ then the greedy algorithm will consistently choose the sub-optimal choice, $B$, and then be forced to pick $A$. In particular, if $f(\{B\}) = 1$ and $f(\{A\}) = f(\{C\}) = 1 -\epsilon$, but $f(\{A,B\}) = 1$ and $f(\{A,C\}) = 2 - 2\epsilon$ (which is readily extended to a submodular function), then the utility gained by the greedy is $\frac{1}{2-2\epsilon}$ times the optimal utility.

We test the practical performance of DPG and PCG in this type of worst-case instance. Although the noise induced by privacy would help DPG overcome this pitfall with some probability, our experiments show that the bound presented above is realized in the experiment. To replicate this instance, we pick three points in Manhattan which mimic this partition structure (with $B$ closest to downtown, and $A$ and $C$ slightly further away) and compare the DPG and the PCG algorithms on a range of dataset sizes. The utility obtained is divided by the number of points in the dataset which is an upper bound for the optimal utility. 

In Figure~\ref{fig:matroid}, we compare the average utilities obtained by PCG (with $\eta=1/7$ and $\delta=1/m^{1.5}$) and DPG with the improved privacy analysis. The error bars at each point mark 1-standard deviation confidence intervals. Although their performances are comparable for small datasets, the improvement of PCG increases as the dataset grows in size. There is high variance due to the random choice of the dataset for each set size, but the separation between the empirical confidence intervals still widens with larger datasets. We find that for these types of worst-case instances, compared to DPG (even with the improved privacy analysis) a significant performance enhancement can be obtained by switching to PCG for the decomposable setting.
\begin{figure}
\centering
	\begin{tikzpicture}
	\begin{axis}[
	title={Utility versus number of data, $\eps = 0.1$},
	xlabel={Number of data $m$},
	ylabel={Utility},
	legend pos=north west,
	ymajorgrids=true,
	xmajorgrids=true,
	grid style=dashed,
	]
	
	\addplot[yellow] plot[ 
	error bars/.cd,
	y dir = both, y explicit
	]
	coordinates {
		(1000,0.7902870000000001)+-(0.00364512,0.00364512)
		(2000,0.79472)+-(0.003403305,0.003403305)
		(3000,0.8084033333333334)+-(0.0026899000000000003,0.0026899000000000003)
		(4000,0.8110925)+-(0.0031621749999999997,0.0031621749999999997)
		(5000,0.8035180000000001)+-(0.0030962,0.0030962)
		(6000,0.8004883333333334)+-(0.0031712166666666665,0.0031712166666666665)
		(7000,0.8104942857142857)+-(0.002902142857142857,0.002902142857142857)
		(8000,0.81078625)+-(0.0029497875,0.0029497875)
		(9000,0.8219733333333333)+-(0.0020228333333333335,0.0020228333333333335)
		(10000,0.810046)+-(0.00291266,0.00291266)
	};
	\addplot[blue] plot[ error bars/.cd,
	y dir = both, y explicit
	]
	coordinates {
		(1000,0.7858419999999999)+-(0.0036990699999999996,0.0036990699999999996)
		(2000,0.791995)+-(0.0034040399999999997,0.0034040399999999997)
		(3000,0.78637)+-(0.0030962333333333335,0.0030962333333333335)
		(4000,0.793175)+-(0.003532,0.003532)
		(5000,0.783772)+-(0.003179,0.003179)
		(6000,0.7886716666666667)+-(0.0032672666666666668,0.0032672666666666668)
		(7000,0.7880228571428571)+-(0.0032394285714285715,0.0032394285714285715)
		(8000,0.78253375)+-(0.003137075,0.003137075)
		(9000,0.7850511111111111)+-(0.0030192,0.0030192)
		(10000,0.781828)+-(0.0030065300000000003,0.0030065300000000003)
	};
	\legend{This work,MBKK}
	
	\end{axis}
	\end{tikzpicture}
	\caption{Empirical performance of PCG and DPG under the partition matroid constraint.}\label{fig:matroid}
\end{figure}
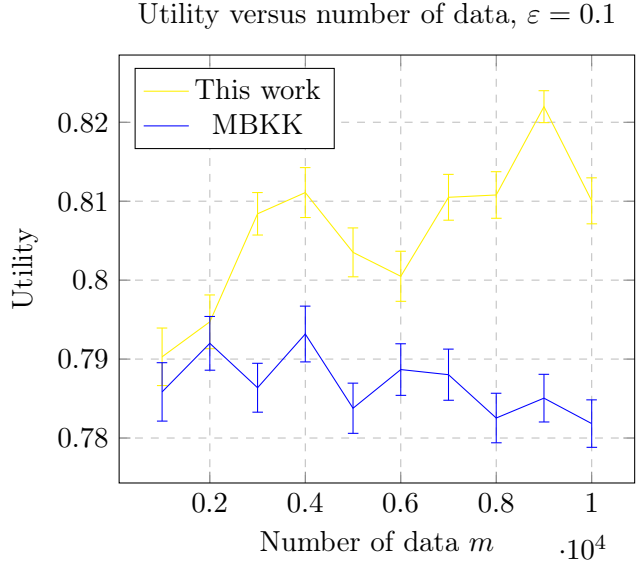